\documentclass[10pt]{article}

\usepackage[utf8]{inputenc}
\usepackage[T1]{fontenc}

\usepackage{epsf}
\usepackage{amsmath}

\allowdisplaybreaks

\usepackage[showframe=false]{geometry}
\usepackage{changepage}

\usepackage{epsfig}
\usepackage{amssymb}

\usepackage{amsthm}
\usepackage{setspace}
\usepackage{cite}
\usepackage{mcite}

\usepackage{algorithmic}  
\usepackage{algorithm}

\usepackage{shadow}
\usepackage{fancybox}
\usepackage{fancyhdr}

\usepackage{color}
\usepackage[usenames,dvipsnames,svgnames,table]{xcolor}
\newcommand{\bl}[1]{\textcolor{blue}{#1}}
\newcommand{\red}[1]{\textcolor{red}{#1}}

\definecolor{mypurple}{rgb}{.4,.0,.5}

\usepackage[hyphens]{url}

\usepackage[colorlinks=true,
            linkcolor=black,
            urlcolor=blue,
            citecolor=purple]{hyperref}

\usepackage{breakurl}

\def\y{{\bf y}}

\def\x{{\bf x}}

\def\x{{\mathbf x}}

\def\u{{\bf u}}

\def\x{{\bf x}}
\def\y{{\bf y}}
\def\z{{\bf z}}
\def\q{{\bf q}}
\def\m{{\bf m}}

\def\c{{\bf c}}

\def\h{{\bf h}}

\def\bg{\bar{\g}}

\def\cH{{\mathcal H}}
\def\cG{{\mathcal G}}

\def\be{\begin{equation}}
\def\ee{\end{equation}}
\def\ba{\left[\begin{array}}
\def\ea{\end{array}\right]}

\def\u{{\bf u}}

\def\x{{\bf x}}
\def\y{{\bf y}}
\def\z{{\bf z}}
\def\q{{\bf q}}

\def\c{{\bf c}}

\def\p{{\bf p}}

\def\1{{\bf 1}}

\def\g{{\bf g}}
\def\0{{\bf 0}}

\def\erfinv{\mbox{erfinv}}
\def\erf{\mbox{erf}}
\def\erfc{\mbox{erfc}}

\def\calX{{\cal X}}
\def\calY{{\cal Y}}

\def\mR{{\mathbb R}}
\def\mN{{\mathbb N}}
\def\mE{{\mathbb E}}
\def\mS{{\mathbb S}}
\def\mP{{\mathbb P}}

\def\lp{\left (}
\def\rp{\right )}

\def\y{{\bf y}}

\def\x{{\bf x}}

\def\x{{\mathbf x}}

\def\u{{\bf u}}

\def\x{{\bf x}}
\def\y{{\bf y}}
\def\z{{\bf z}}
\def\q{{\bf q}}

\def\c{{\bf c}}

\def\h{{\bf h}}

\def\cH{{\cal H}}

\def\be{\begin{equation}}
\def\ee{\end{equation}}
\def\ba{\left[\begin{array}}
\def\ea{\end{array}\right]}

\def\u{{\bf u}}

\def\x{{\bf x}}
\def\y{{\bf y}}
\def\z{{\bf z}}
\def\q{{\bf q}}

\def\c{{\bf c}}

\def\p{{\bf p}}

\def\({\left (}
\def\){\right )}

\def\1{{\bf 1}}
\def\m{{\bf m}}
\def\q{{\bf q}}

\def\g{{\bf g}}
\def\0{{\bf 0}}

\def\cX{{\mathcal X}}
\def\cY{{\mathcal Y}}

\def\erfinv{\mbox{erfinv}}

\usepackage{xcolor}
\usepackage{color}

\definecolor{darkgreen}{rgb}{0, 0.4,0}

\definecolor{purplebrown}{rgb}{0.5,0.1,0.6}

\definecolor{ultclupcol}{rgb}{0.1,0.5,0.5}

\definecolor{mytrycolor}{rgb}{0.5,0.7,0.2}

\definecolor{ultclupcola}{rgb}{.5,0,.5}

\definecolor{shadebrown}{rgb}{0.1,0.1,0.9}
\definecolor{lightblue}{rgb}{0.2,0,1}

\usepackage{fancybox}
\usepackage{graphicx}
\usepackage{epstopdf}
\usepackage{epsfig}
\usepackage{wrapfig}
\usepackage{subfigure}

\usepackage{xcolor}
\usepackage{tcolorbox}
\tcbuselibrary{skins}

\newtcbox{\xmybox}{on line,
arc=7pt,
before upper={\rule[-3pt]{0pt}{10pt}},boxrule=0pt,
boxsep=0pt,left=6pt,right=6pt,top=0pt,bottom=0pt,enhanced, coltext=blue, colback=white!10!yellow}

\newtcbox{\xmyboxa}{on line,
arc=7pt,
before upper={\rule[-3pt]{0pt}{10pt}},boxrule=0pt,
boxsep=0pt,left=6pt,right=6pt,top=0pt,bottom=0pt,enhanced, colback=white!10!yellow}

\newtcbox{\xmyboxb}{on line,
arc=7pt,
before upper={\rule[-3pt]{0pt}{10pt}},boxrule=1pt,colframe=darkgreen!100!blue,
boxsep=0pt,left=6pt,right=6pt,top=0pt,bottom=0pt,enhanced, colback=white!10!yellow}

\newtcbox{\xmyboxc}{on line,
arc=7pt,
before upper={\rule[-3pt]{0pt}{10pt}},boxrule=.7pt,colframe=blue!100!blue,
boxsep=0pt,left=6pt,right=6pt,top=0pt,bottom=0pt,enhanced, coltext=blue, colback=white!10!yellow}

\newtcbox{\xmytboxa}{on line,
arc=7pt,
before upper={\rule[-3pt]{0pt}{10pt}},boxrule=.0pt,colframe=pink!50!yellow,
boxsep=0pt,left=6pt,right=6pt,top=0pt,bottom=0pt,enhanced, coltext=white, colback=blue!40!red}

\newtcbox{\xmytboxb}{on line,
arc=7pt,
before upper={\rule[-3pt]{0pt}{10pt}},boxrule=.0pt,colframe=pink!50!yellow,
boxsep=0pt,left=6pt,right=6pt,top=0pt,bottom=0pt,enhanced, coltext=white, colback=white!40!green}

\makeatletter
\newcommand\subsubsubsection{\@startsection{paragraph}{4}{\z@}{-2.5ex\@plus -1ex \@minus -.25ex}{1.25ex \@plus .25ex}{\normalfont\normalsize\bfseries}}
\newcommand\subsubsubsubsection{\@startsection{subparagraph}{5}{\z@}{-2.5ex\@plus -1ex \@minus -.25ex}{1.25ex \@plus .25ex}{\normalfont\normalsize\bfseries}}
\makeatother

\newtheorem{theorem}{Theorem}

\newtheorem{corollary}{Corollary}

\begin{document}

\begin{singlespace}

\title {Capacity of the Hebbian-Hopfield network associative memory  
}
\author{
\textsc{Mihailo Stojnic
\footnote{e-mail: {\tt flatoyer@gmail.com}} }}
\date{}
\maketitle

\centerline{{\bf Abstract}} \vspace*{0.1in}

In \cite{Hop82}, Hopfield introduced a \emph{Hebbian} learning rule based neural network model and suggested how it can efficiently operate as an associative memory. Studying random binary patterns, he also uncovered that, if a small fraction of errors is tolerated in the stored patterns retrieval, the capacity of the network (maximal number of memorized patterns, $m$) scales linearly with each pattern's size, $n$. Moreover, he famously predicted $\alpha_c=\lim_{n\rightarrow\infty}\frac{m}{n}\approx 0.14$. We study this very same scenario with two most famous pattern's basins of attraction notions: \textbf{\emph{(i)}} The AGS one from  \cite{AmiGutSom85} which relies on the existence of a local energy minimum around aimed pattern; and \textbf{\emph{(ii)}} The NLT one from \cite{Newman88,Louk94,Louk94a,Louk97,Tal98} which relies on the existence of a firm energy barrier around patterns. Relying on the \emph{fully lifted random duality theory} (fl RDT) from \cite{Stojnicflrdt23}, we obtain the following explicit closed form capacity characterizations on the first level of lifting:
\begin{eqnarray}
 \alpha_c^{(AGS,1)}  & = &  \left ( \max_{\delta\in \left ( 0,\frac{1}{2}\right ) }\frac{1-2\delta}{\sqrt{2} \mbox{erfinv} \left ( 1-2\delta\right )} -   \frac{2}{\sqrt{2\pi}} e^{-\left ( \mbox{erfinv}\left ( 1-2\delta \right )\right )^2}\right )^2 \approx \mathbf{0.137906} \nonumber \\
  \alpha_c^{(NLT,1)} &  =  & \frac{\mbox{erf}(x)^2}{2x^2}-1+\mbox{erf}(x)^2 \approx \mathbf{0.129490}, \quad
1-\mbox{erf}(x)^2- \frac{2\mbox{erf}(x)e^{-x^2}}{\sqrt{\pi}x}+\frac{2e^{-2x^2}}{\pi}=0.\nonumber
 \label{eq:abs2}
\end{eqnarray}
To maximally utilize fl RDT, one also needs to perform substantial numerical work on higher levels of lifting. After doing so, we obtain $\alpha_c^{(AGS,2)} \approx \mathbf{0.138186}$ and $\alpha_c^{(NLT,2)} \approx \mathbf{0.12979}$, effectively uncovering a remarkably fast lifting convergence  (with the relative improvement no better than $\sim 0.1\%$ already on the \textbf{\emph{second}} level of lifting). Moreover, the obtained AGS characterizations exactly match those obtained based on replica symmetry methods of \cite{AmiGutSom85} and the corresponding symmetry breaking ones of \cite{SteKuh94}. The NLT ones are substantially higher than the previously best known ones of \cite{Newman88,Louk94,Louk94a,Louk97,Tal98}.

\vspace*{0.25in} \noindent {\bf Index Terms: Hopfield network; Hebbian rule; Associative memory capacity; Random duality}.

\end{singlespace}

\section{Introduction}
\label{sec:back}

Following an initial introduction in \cite{PasFig77} (where they were viewed as a simpler form of spin glasses), the present day Hopfield neural networks models became prevalently popular when their connection to the auto-associative memories was proposed in \cite{Hop82}. As is by now well known, the main idea revolves around an iterative pattern searching dynamics that aims to converge to a local optimum which corresponds to or is in a vicinity of a pattern to be stored/memorized. Similarly to any other memory concept, two key questions immediately position themselves as of critical importance: 1) What is the memory capacity (maximum allowable number of stored patterns) and how does it behave/scale as a function of other system parameters; and 2) How efficiently a practical dynamics realization can approach such a capacity. We proceed along these lines and study the first question in detail. Given that the associative capacity of the Hopfield nets has been a very popular research subject since the appearance of the  Hopfield's introductory paper \cite{Hop82}, a large body of relevant work is already available in the existing literature. Before proceeding with the exposition of our own approach and results, we briefly survey the already achieved key milestones.

\subsection{Relevant prior work}
\label{sec:prior}

Relying on the Hebbian learning principle/rule \cite{Hebb49}, Hopfiled in \cite{Hop82} proposed a basic neural associative dynamics model. As the numerical evidence suggested, the dynamics seemed to have the ability to converge to a pattern of $n$ binary neurons (state of the network) that either  exactly matched one of the previously stored patterns or was in a close proximity of it. Viewing the maximum number of the stably stored/retrieved patterns,$m$, as the associative capacity of the network, Hopfield proceeded further, utilized a simple Gaussian noise argument, and gave rough capacity estimates. Assuming random iid binary patterns, his noise based reasoning suggested that the exact pattern retrieval (convergence of the dynamics to one of the stored patterns), happens (with probability going to 1 as $n\rightarrow \infty$) provided that $m\approx \frac{n}{2\log(n)}$. While this result is already very encouraging, an even better one was awaiting. Namely, Hopfield noted that a small fraction of errors in retrieved patterns, allows even for a \emph{linear} capacity scaling. Moreover, his noise based reasoning gave prediction for the linear proportionality $\alpha\triangleq \lim_{n\rightarrow \infty}\frac{m}{n}\approx 0.14$ with only $0.4\%$ of errors, which happened to be in an solid agrement with the numerical predictions.

Hopfield's initial considerations provided a very strong foundation and generated a large amount of research regarding various aspects of these topics over the last four decades. We below discuss some of the works that are the most relevant to our own and that have made a very strong analytical progress. To properly highlight some of the key analytical aspects of the problem, we find it convenient to distinguish four separate streams of relevant research and within each of them focus on the most representative results. As almost all of the relevant results that we revisit below rely on the same random binary patterns statistical context already mentioned above, we skip reemphasizing it. Along the same lines, we adopt the terminology convention of not explicitly mentioning that the known results stated below hold with probability going to 1 as $n\rightarrow\infty$ (the very same convention is later on  adopted in the presentation of our own results as well).

\textbf{\emph{(i)}} \underline{\emph{Exact retrieval:}} Following the initial Hopfield's noise based considerations, \cite{AmiGutSom85} analyzed the memory capacity problem utilizing the statistical physics replica theory methods. Assuming the replica symmetry, the very same $m\approx \frac{n}{2\log(n)}$ capacity estimate for retrieving \emph{exactly} the stored patterns was obtained. A bit later on, \cite{MPRV87} provided the very first mathematically rigorous results in this direction, which (depending on the type of the probabilistic statements) matched the \cite{Hop82,AmiGutSom85} sublinear exact retrieval scaling. The very same  $\frac{n}{2\log(n)}$ result obtained in \cite{AmiGutSom85} through the replica symmetry analysis was shown as a rigorous lower bound in \cite{FST00}. In a somewhat related line of work, \cite{KomPat88} extended the results of \cite{MPRV87} and showed that each of the stored patterns has a nonzero basin of attraction.  \cite{MPRV87} also provided a strong numerical evidence that the patterns are unlikely to be \emph{exactly} retrievable if the capacity were to scale stronger than $\sim \frac{n}{\log(n)}$.

\textbf{\emph{(ii)}} \underline{\emph{Approximate retrieval -- AGS basin of attraction:}} As already noted in Hopfield's considerations, one of the key properties of the underlying dynamics is its ability to achieve a \emph{linear} capacity scaling if approximate retrieval is allowed. A combination of Hopfield rough analytical estimates and strong algorithmic numerical evidence suggested that the linear proportionality should be around $\alpha\triangleq \lim_{n\rightarrow \infty}\frac{m}{n}\approx 0.14$. Relying on the existence of an energy well around each of the stored patterns (Amit-Gutfreund-Sompolinsky (AGS) basin of attraction) and the convergence of the Hopfield dynamics to a local energy (objective) minima of such a well, the replica symmetry analysis of \cite{AmiGutSom85} provided a very precise estimate $\alpha\approx 0.137906$. Moreover, the discussion of \cite{AmiGutSom85} suggested that further replica symmetry breaking (rsb) corrections should not necessarily be quantitatively substantial. Such a suggestion was indeed confirmed in \cite{SteKuh94}, where it was obtained, relying on 1rsb, that $\alpha\approx 0.138186$ and relying on 2rsb, that $\alpha\approx 0.138187$. As far as the mathematically rigorous AGS basin related results are concerned, \cite{FST00} showed that $\alpha\geq 0.113$.

\textbf{\emph{(iii)}} \underline{\emph{Approximate retrieval -- NLT basin of attraction:}} A strong rigorous progress, however, has been made considering a different type of attraction basin. Namely, in \cite{Newman88,Louk94,Louk94a,Louk97} an alternative (to which we refer as the NLT basin) to the AGS was analyzed. Such alternative assumes that in addition to the existence of the AGS type of energy well around each pattern, one also insists that the well's local energy maximum is not smaller than the one achieved by the pattern itself. Clearly, such a basin is a bit more restrictive and clearly its existence implies the existence of the AGS one as well. That also means that all the capacity results for the NLT basin are automatically lower bounds for the corresponding AGS related ones. In \cite{Newman88}, it was shown that $\alpha\geq 0.056$ for the NLT basin. This results was then improved to $\alpha\geq 0.071$ in \cite{Louk94,Louk94a}  and to $\alpha\geq 0.08$ in \cite{Tal98}.

\textbf{\emph{(iv)}} \underline{\emph{Non-Hebbian learning rules:}} It is also interesting to note that the capacity can grow even beyond the linear one if one is willing to deviate from the standard Hebbian learning principle. If, instead of the Hebbian rule which assumes the quadratic (polynomial of degree 2) objective, one relies on an objective based on a general overlap power (polynomial of degree) $p$, the capacity of the \emph{exact} retrieval scales as $m\sim \frac{n^{p-1}}{\log(n)}$  and the capacity of the \emph{approximate} retrieval scales as $m\sim n^{p-1}$   \cite{GarMult87,AAAABGLP23,KroHop16} (clearly, for $p=2$ these estimates match the ones discussed above). If, on the other hand, one relies on the exponential interactions, the capacity can actually grow even exponentially \cite{DHLUV17,Rametal21,Vasetal17,LucMez24}. While these results allow for a substantial increase in the analytically/algorithmically achievable associative capacity, the connection of the underlying learning principles to the neurological behavior is not necessarily as clear as in the Hebbian rule based models.

\subsection{Our contributions}
\label{sec:contrib}

After recognizing the connection between the Hopfiled associative memory models and \emph{bilinearly indexed} (bli) random processes, we start by utilizing a strong progress achieved recently in studying bli's in \cite{Stojnicsflgscompyx23,Stojnicnflgscompyx23}. We make a particular use of the \emph{fully lifted} random duality theory (fl RDT) established in \cite{Stojnicflrdt23} that first allows us create a powerful generic framework for the analysis of the Hopfield dynamics and then to also obtain associated memory capacity characterizations. For the utilization of the fl RDT to become practically relevant, a substantial amount of the underlying numerical work needs to be performed as well. Performing such a work more often than not requires a rather strong effort. Relying on a remarkable set of closed form explicit analytical relations among the key lifting parameters turns out to be extremely helpful in that regard. In addition to being helpful with the numerical work, these relations also provide a direct insight into a beautiful structuring of the parametric interconnections. The most important of all, they eventually allow us to uncover that the capacity characterizations, obtained already on the \emph{\textbf{second}} level of the \emph{full} lifting (2-sfl RDT), exhibit a remarkably rapid convergence with relative improvements of the order of $\sim 0.1\%$.

\section{Mathematical setup}
 \label{sec:mathsetup}

We assume the standard Hopfield model with $m$ binary patterns to be memorized. The $i$-th pattern will be denoted by $\bg^{(i)}\in \mS_H^n,i=1,2,\dots,m,$ where $\mS_H^n$ is the $n$-dimensional Hamming sphere of radius 1, i.e. $\mS_H^n\triangleq\left \{-\frac{1}{\sqrt{n}},\frac{1}{\sqrt{n}}\right \}^n$. Speaking in a high-dimensional geometry terminology, $\mS_H^n$ is the set of all corners of the unit $n$-dimensional hypercube. Also, we denote by $\bg_j^{(i)},j=1,2,\dots,n,$ the $j$-th component of $\bg^{(i)}$ and assume the typical statistical context where $\bg_j^{(i)}$ are iid $\pm 1$ Bernoulli random variables. For the writing convenience, we also find it convenient to set
\begin{eqnarray}
\cG\triangleq \sum_{i=1}^{m} \bg^{(i)}\lp\bg^{(i)}\rp^T. \label{eq:ex0}
\end{eqnarray}
The following is then the associated Hopfield sequential retrieval dynamics
\begin{eqnarray}
\bar{\x}_i(t+1)=\mbox{sign}\lp\sum_{j=1,j\neq i}^{n}\cG_{i,j}\bar{\x}_j(t)\rp, \label{eq:ex0a0}
\end{eqnarray}
where (here and throughout the paper) $\cG_{i,j}$ stands for the $ij$-th component of $\cG$ and $\cG_{i,:}$ is the $i$-th row of $\cG$. If one defines the energy objective as
\begin{eqnarray}
\cH(\x)\triangleq-\x^T\cG\x, \label{eq:ex0a1}
\end{eqnarray}
the above dynamics, (\ref{eq:ex0a0}), converges to its a local minimum. Let $\bar{\x}$ be the converging point of (\ref{eq:ex0a0}), i.e. let
\begin{eqnarray}
\bar{\x}=\lim_{t\rightarrow \infty} \bar{\x}(t), \label{eq:ex0a2}
\end{eqnarray}
and let the corresponding achieved value of the local energy minimum be
\begin{eqnarray}
\cH(\bar{\x})=-\bar{\x}^T\cG\bar{\x}. \label{eq:ex0a3}
\end{eqnarray}
One then also has
\begin{eqnarray}
\cH(\bar{\x})\leq \min_{\|\x-\bar{\x}\|_2=\frac{2}{\sqrt{n}},\x\in S_H^n} -\x^T\cG\x. \label{eq:ex0a4}
\end{eqnarray}
The main idea behind the energy objective in (\ref{eq:ex0a1}) is that, provided that the patterns $\bg$ are ``sufficiently orthogonal'', the dynamics might actually converge so that it either exactly or approximately matches one of the stored patterns. In other words, it might happen that $\bar{\x}^T\bg^{(i)}=1$ or $\bar{\x}^T\bg^{(i)}=1-2\delta$, where $\delta$ is a small real number. In fact, as the noise arguments of \cite{Hop82} and the replica symmetry ones of \cite{AmiGutSom85} suggested, and the rigorous ones of \cite{MPRV87,FST00} confirmed, one indeed has $\bar{\x}^T\bg^{(i)}=1$ provided that $m\leq\frac{n}{2\log(n)}$. On the other hand, a bit looser scenario $\bar{\x}^T\bg^{(i)}=1-2\delta$ allows even for (a higher) $m$ that grows aa a linear function of $n$. To be able to precisely state this, we distinguish two different sub-scenarios depending of the choice of the basin of attraction.

\subsection{AGS basin of attraction}
 \label{sec:agsbasin}

In \cite{AmiGutSom85}, the Hopfield dynamics convergence to a local minimum that approximately matches one of the patterns to be retrieved and the corresponding associative capacity are formulated relying on the existence of appropriately defined energy wells. Taking $\g^{(1)}$ as the targeted pattern of retrieval, the capacity (in AGS sense) is defined as the largest $\alpha\triangleq \lim_{n\rightarrow\infty} \frac{m}{n}$ such that for some (small) $\delta>0$
\begin{eqnarray}
\hat{\x}^{(AGS)}(\delta) \triangleq \mbox{arg min}_{\|\x-\bg^{(1)}\|_2^2=4\delta,\x\in S_H^n} -\x^T\cG\x. \label{eq:ex0a5}
\end{eqnarray}
and
\begin{eqnarray}
\cH(\bar{\x}) \leq  \min \lp \cH(\bg^{(1)}), \cH(\hat{\x}^{(AGS)}(\delta))  \rp. \label{eq:ex0a6}
\end{eqnarray}
As mentioned earlier, the replica symmetry arguments of \cite{AmiGutSom85} gave $\alpha\approx 0.1379056$ whereas the corresponding 1rsb and 2rsb ones gave $\alpha\approx 0.138186$ and $\alpha\approx 0.138187$, respectively. On the other hand, the rigorous lower bound of \cite{FST00} gave a significantly lower capacity estimate $\alpha\geq 0.113$. Given that the numerical simulations suggest that indeed $\alpha\approx 0.14$ this leaves a substantial rigorous theory -- replica predictions gap to bridge.

\subsection{NLT basin of attraction}
 \label{sec:nlbasin}

A slightly different basin of attraction was considered in \cite{Newman88,Louk94,Louk94a,Louk97,Tal98}. One again relies on the existence of the energy wells but utilizes them in an analytically possibly more convenient way.
Taking again, without loss of generality, $\g^{(1)}$ as the targeted pattern, the capacity (in NLT sense) is defined as the largest $\alpha\triangleq \lim_{n\rightarrow\infty} \frac{m}{n}$ such that for some (small) $\delta>0$
\begin{eqnarray}
\hat{\x}^{(NLT)}(\delta) \triangleq \mbox{arg min}_{\|\x-\bg^{(1)}\|_2^2=4\delta,\x\in S_H^n} -\x^T\cG\x. \label{eq:ex0a7}
\end{eqnarray}
and
\begin{eqnarray}
\cH(\bar{\x}) \leq  \min \lp \cH(\bg^{(1)}), \cH(\hat{\x}^{(NLT)}(\delta))  \rp \qquad \mbox{and} \qquad \cH(\bg^{(1)})< \cH(\hat{\x}^{(NLT)}(\delta)). \label{eq:ex0a8}
\end{eqnarray}
While in writing the difference between this basin of attraction notion and the one from  (\ref{eq:ex0a5}) and (\ref{eq:ex0a6})  seems rather marginal, in analytical terms it is actually significant. It allows one to switch the focus on the second portion of the condition in (\ref{eq:ex0a8}) which is typically analytically easier to handle. Consequently the known rigorous results are a bit more diverse as well. As stated earlier, it was first shown in  \cite{Newman88} that $\alpha\geq 0.056$. This was then improved to $\alpha\geq 0.071$ in \cite{Louk94,Louk94a} and to $\alpha\geq 0.08$ in \cite{Tal98}. It is also not that difficult to see that (\ref{eq:ex0a7}) and (\ref{eq:ex0a8}) are more restrictive than (and certainly imply) (\ref{eq:ex0a5}) and (\ref{eq:ex0a6}). This then also means that the capacity estimates obtained based on the NLT basin of attraction are for sure lower bounds on the corresponding ones obtained relying on the AGS basin.

No matter which of the above attractions basins one adopts, it is clear that the optimization problems from (\ref{eq:ex0a5}) or (\ref{eq:ex0a7}) play a critical role in determining the corresponding associative memory capacity. In what follows, we therefore focus precisely on these optimization programs.

\subsection{Statistical associative memory capacity}
 \label{sec:statcap}

We first set
\begin{eqnarray}
\xi_1(\delta) \triangleq \lim_{n\rightarrow\infty}\frac{1}{n}\min_{\|\x-\bg^{(1)}\|_2^2=4\delta,\x\in S_H^n} -\x^T\cG\x, \label{eq:ex0a9}
\end{eqnarray}
and then, after noting that for $\x\in S_H^n$, $\|\x-\bg^{(1)}\|_2^2=4\delta$ implies $\x^T\bg^{(1)}=1-2\delta$, write
\begin{eqnarray}
\xi_1(\delta) = -\lp 1-2\delta\rp^2 + \lim_{n\rightarrow\infty}\frac{1}{n}
\min_{\|\x-\bg^{(1)}\|_2^2=4\delta,\x\in S_H^n} -\sum_{i=2}^{m}\x^T\g^{(i)} \lp \g^{(i)}\rp^T\x. \label{eq:ex0a10}
\end{eqnarray}
After setting
\begin{eqnarray}
 G\triangleq \begin{bmatrix}
               \lp\g^{(2)}\rp^T \\ \lp\g^{(3)}\rp^T \\ \vdots \\ \lp\g^{(m)}\rp^T
             \end{bmatrix}, \label{eq:ex0a11}
\end{eqnarray}
one can rewrite (\ref{eq:ex0a10}) as
\begin{eqnarray}
\xi_1(\delta)
& = & -\lp 1-2\delta\rp^2 + \lim_{n\rightarrow\infty} \frac{1}{n}
\min_{\|\x-\bg^{(1)}\|_2^2=4\delta,\x\in S_H^n} - \x^TG^TG\x \nonumber \\
& = & -\lp 1-2\delta\rp^2 + \lim_{n\rightarrow\infty} \frac{1}{n}
\min_{\|\x-\bg^{(1)}\|_2^2=4\delta,\x\in S_H^n} - \|G\x\|_2^2 \nonumber \\
& = & -\lp 1-2\delta\rp^2
- \lim_{n\rightarrow\infty} \frac{1}{n} \max_{\|\x-\bg^{(1)}\|_2^2=4\delta,\x\in S_H^n}  \|G\x\|_2^2 \nonumber \\
& = & -\lp 1-2\delta\rp^2 -\xi(\delta)^2, \label{eq:ex0a12}
\end{eqnarray}
where
\begin{eqnarray}
\xi(\delta) \triangleq \lim_{n\rightarrow\infty} \frac{1}{\sqrt{n}} \max_{\|\x-\bg^{(1)}\|_2^2=4\delta,\x\in S_H^n}  \|G\x\|_2. \label{eq:ex0a13}
\end{eqnarray}
The  problem from (\ref{eq:ex0a13}) can be rewritten in the following way
\begin{eqnarray}
\xi(\delta) \triangleq \lim_{n\rightarrow\infty}\frac{1}{\sqrt{n}} \max_{\x\in \cX(\delta)}\max_{\y\in\mS^m}  \y^TG\x, \label{eq:ex0a14}
\end{eqnarray}
where $\mS^m$ is the unit $m$-dimensional Euclidean sphere, and for $\delta\in [0,1]$
\begin{eqnarray}
\cX(\delta) \triangleq \left \{\x | \quad \|\x-\bg^{(1)}\|_2^2=4\delta,\x\in S_H^n\right \}. \label{eq:ex0a15}
\end{eqnarray}
Even though we assumed at the beginning that the patterns are comprised of iid binary random variables, everything that we have written above holds deterministically, i.e. for any $G$ and $\g^{(1)}$. To make the exposition neater, in what follows, we view $G$ as being comprised of iid standard normals (concentrations and the Lindeberg variant of the central limit theorem allow for such a switch without a substantial change in the final results; while there are many ways as to how the Lindeberg principle can be implemented, we find as particularly elegant and quick the one from \cite{Chatterjee06}). Moreover, to further facilitate the writing and exposition, we, without loss of generality, assume that $\g^{(1)}=\frac{\1}{\sqrt{n}}$, i.e., we assume that $\g^{(1)}$ is an $n$-dimensional column vector of all ones scaled so that it has a unit Euclidean norm.

After recognizing that $\lim_{n\rightarrow\infty}\frac{1}{n}\cH\lp \bg^{(1)}\rp=\xi_1(0)$, we can formally define the AGS and the NLT basins of attractions associative memory \emph{capacity} in a large dimensional statistical context in the following way
 \begin{eqnarray}
\alpha & = &    \lim_{n\rightarrow \infty} \frac{m}{n}  \nonumber \\
\alpha_c^{(AGS)} & \triangleq & \max \{\alpha |\hspace{.08in} \exists \delta>\delta_1>0, \lim_{n\rightarrow\infty}
\mP_G\lp \xi_1(\delta_1)\leq \min( \xi_1(0),\xi_1(\delta))\rp\longrightarrow 1\} \nonumber \\
\alpha_c^{(NLT)} & \triangleq & \max \{\alpha |\hspace{.08in} \exists \delta>\delta_1>0, \lim_{n\rightarrow\infty}
\mP_G\lp \xi_1(\delta_1)\leq \xi_1(0) < \xi_1(\delta)\rp\longrightarrow 1\}.
  \label{eq:ex0a16}
\end{eqnarray}
We adopt the convention that the subscripts next to $\mP$ and $\mE$, whenever appear throughout the paper, denote the randomness with respect to which the statistical evaluation is taken. On the other hand, the subscripts are left unspecified when it is clear from the context what the underlying randomness is.

\subsection{Free energy correspondence}
 \label{sec:feeeng}

It is rather clear from the above presentation (and particularly so from (\ref{eq:ex0a16})) that studying $\xi_1(\delta)$ is of critical importance for determining the capacity. Moreover, from (\ref{eq:ex0a12}), one can actually observe that the most relevant object for studying is in fact $\xi(\delta)$. In that regard, the representation from (\ref{eq:ex0a14}) turns out to be quite relevant.
Namely, as recognized on numerous occasions when studying various random feasibility problems (including many perceptron instances)
\cite{StojnicGardGen13,StojnicGardSphErr13,StojnicGardSphNeg13,StojnicDiscPercp13,Stojnicbinperflrdt23}, studying the representation from (\ref{eq:ex0a14}), in a way, directly corresponds to studying statistical physics objects called \emph{free energies}. For the completeness, we below sketch the most important contours of such a correspondence and refer for a more detailed discussion to \cite{StojnicGardGen13,StojnicGardSphErr13,StojnicGardSphNeg13,StojnicDiscPercp13,Stojnicbinperflrdt23}.

Free energies are well known and almost unavoidable objects in many statistical physics considerations. To introduce their mathematically representation of interest here, we start with the following, so-called, \emph{bilinear Hamiltonian}
\begin{equation}
\cH_{sq}(G)= \y^TG\x,\label{eq:ham1}
\end{equation}
and its corresponding partition function
\begin{equation}
Z_{sq}(\beta,G)=\sum_{\x\in\cX} \sum_{\y\in\cY}e^{\beta\cH_{sq}(G)}.  \label{eq:partfun}
\end{equation}
For the time being $\cX$ and $\cY$ are taken as general sets. However, fairly soon, specializations, $\cX=\cX(\delta)$ and $\cY=\mS^m$,  necessary for our considerations here will take place as well. One then further has for the corresponding  thermodynamic limit of the average free energy
\begin{eqnarray}
f_{sq}(\beta) & = & \lim_{n\rightarrow\infty}\frac{\mE_G\log{(Z_{sq}(\beta,G)})}{\beta \sqrt{n}}
=\lim_{n\rightarrow\infty} \frac{\mE_G\log\lp \sum_{\x\in\cX} \sum_{\y\in\cY}e^{\beta\cH_{sq}(G)} \rp}{\beta \sqrt{n}} \nonumber \\
& = &\lim_{n\rightarrow\infty} \frac{\mE_G\log\lp \sum_{\x\in\cX}  \sum_{\y\in\cY}e^{\beta\y^TG\x)} \rp}{\beta \sqrt{n}}.\label{eq:logpartfunsqrt}
\end{eqnarray}
The ground state energy is a special case obtained by considering the so-called ``zero-temperature'' ($T\rightarrow 0$ or  $\beta=\frac{1}{T}\rightarrow\infty$) regime
\begin{eqnarray}
f_{sq}(\infty)   \triangleq    \lim_{\beta\rightarrow\infty}f_{sq}(\beta) & = &
\lim_{\beta,n\rightarrow\infty}\frac{\mE_G\log{(Z_{sq}(\beta,G)})}{\beta \sqrt{n}}
=
 \lim_{n\rightarrow\infty}\frac{\mE_G \max_{\x\in\cX}   \max_{\y\in\cY} \y^TG\x}{\sqrt{n}}.
  \label{eq:limlogpartfunsqrta0}
\end{eqnarray}
It is not that difficult to see that (\ref{eq:limlogpartfunsqrta0}) is directly related to (\ref{eq:ex0a14}). At the same time, this also implies that $f_{sq}(\infty)$ is very tightly connected to $\xi_1(\delta)$, which then means that studying and understanding  $f_{sq}(\infty)$ is of critical importance in characterizing the statistical capacity from (\ref{eq:ex0a16}). This is basically exactly what will happen in the sections that follow below. Before proceeding further, we find it useful to emphasize that studying $f_{sq}(\infty)$ directly is usually not very easy. We therefore rely on studying first $f_{sq}(\beta)$ for a general $\beta$  and then on specializing the obtained general results to the ground state, $\beta\rightarrow\infty$, regime. To ease the exposition, we, on occasion, neglect some analytical details which paly no significant role in the ultimate ground state considerations.

\section{Connection to bli random processes and sfl RDT}
\label{sec:randlincons}

A key observation that is of fundamental importance for everything that follows is that the free energy from (\ref{eq:logpartfunsqrt}),
\begin{eqnarray}
f_{sq}(\beta) & = &\lim_{n\rightarrow\infty} \frac{\mE_G\log\lp \sum_{\x\in\cX} \sum_{\y\in\cY}e^{\beta\y^TG\x)}\rp}{\beta \sqrt{n}},\label{eq:hmsfl1}
\end{eqnarray}
 can be viewed as a function of \emph{bilinearly indexed} (bli) random process $\y^TG\x$. Such a recognition then allows us to further connect $f_{sq}(\beta)$ and the bli related results of \cite{Stojnicsflgscompyx23,Stojnicnflgscompyx23,Stojnicflrdt23}. To do so, we closely follow the machinery of \cite{Stojnichopflrdt23,Stojnicbinperflrdt23} and start with a collection of necessary technical definitions. For $r\in\mN$, $k\in\{1,2,\dots,r+1\}$, sets $\cX\subseteq \mR^n$ and $\cY\subseteq \mR^m$, function $f_S(\cdot):\mR^n\rightarrow R$, real scalars $x$, $y$, and $s$  such that $x>0$, $y>0$, and $s^2=1$,   vectors $\p=[\p_0,\p_1,\dots,\p_{r+1}]$, $\q=[\q_0,\q_1,\dots,\q_{r+1}]$, and $\c=[\c_0,\c_1,\dots,\c_{r+1}]$ such that
 \begin{eqnarray}\label{eq:hmsfl2}
1=\p_0\geq \p_1\geq \p_2\geq \dots \geq \p_r\geq \p_{r+1} & = & 0 \nonumber \\
1=\q_0\geq \q_1\geq \q_2\geq \dots \geq \q_r\geq \q_{r+1} & = &  0,
 \end{eqnarray}
$\c_0=1$, $\c_{r+1}=0$, and ${\mathcal U}_k\triangleq [u^{(4,k)},\u^{(2,k)},\h^{(k)}]$  such that the elements of  $u^{(4,k)}\in\mR$, $\u^{(2,k)}\in\mR^m$, and $\h^{(k)}\in\mR^n$ are iid standard normals, we first set
  \begin{eqnarray}\label{eq:fl4}
\psi_{S,\infty}(f_{S},\calX,\calY,\p,\q,\c,x,y,s)  =
 \mE_{G,{\mathcal U}_{r+1}} \frac{1}{n\c_r} \log
\lp \mE_{{\mathcal U}_{r}} \lp \dots \lp \mE_{{\mathcal U}_3}\lp\lp\mE_{{\mathcal U}_2} \lp \lp Z_{S,\infty}\rp^{\c_2}\rp\rp^{\frac{\c_3}{\c_2}}\rp\rp^{\frac{\c_4}{\c_3}} \dots \rp^{\frac{\c_{r}}{\c_{r-1}}}\rp, \nonumber \\
 \end{eqnarray}
where
\begin{eqnarray}\label{eq:fl5}
Z_{S,\infty} & \triangleq & e^{D_{0,S,\infty}} \nonumber \\
 D_{0,S,\infty} & \triangleq  & \max_{\x\in\cX,\|\x\|_2=x} s \max_{\y\in\cY,\|\y\|_2=y}
 \lp \sqrt{n} f_{S}
+\sqrt{n}  y    \lp\sum_{k=2}^{r+1}c_k\h^{(k)}\rp^T\x
+ \sqrt{n} x \y^T\lp\sum_{k=2}^{r+1}b_k\u^{(2,k)}\rp \rp \nonumber  \\
 b_k & \triangleq & b_k(\p,\q)=\sqrt{\p_{k-1}-\p_k} \nonumber \\
c_k & \triangleq & c_k(\p,\q)=\sqrt{\q_{k-1}-\q_k}.
 \end{eqnarray}
With all the above definitions set, we can proceed by recalling on the following theorem -- clearly, one of most fundamental components of sfl RDT.
\begin{theorem} \cite{Stojnicflrdt23}
\label{thm:thmsflrdt1}  Consider large $n$ linear regime with  $\alpha\triangleq \lim_{n\rightarrow\infty} \frac{m}{n}$, remaining constant as  $n$ grows. Let $\cX\subseteq \mR^n$ and $\cY\subseteq \mR^m$ be two given sets and let the elements of  $G\in\mR^{m\times n}$
 be i.i.d. standard normals. Assume the complete sfl RDT frame from \cite{Stojnicsflgscompyx23} and consider a given function $f(\y):R^m\rightarrow R$. Set
\begin{align}\label{eq:thmsflrdt2eq1}
   \psi_{rp} & \triangleq - \max_{\x\in\cX} s \max_{\y\in\cY} \lp f(\y)+\y^TG\x \rp
   \qquad  \mbox{(\bl{\textbf{random primal}})} \nonumber \\
   \psi_{rd}(\p,\q,\c,x,y,s) & \triangleq    \frac{x^2y^2}{2}    \sum_{k=2}^{r+1}\Bigg(\Bigg.
   \p_{k-1}\q_{k-1}
   -\p_{k}\q_{k}
  \Bigg.\Bigg)
\c_k
  - \psi_{S,\infty}(f(\y),\calX,\calY,\p,\q,\c,x,y,s) \hspace{.03in} \mbox{(\bl{\textbf{fl random dual}})}. \nonumber \\
 \end{align}
Let $\hat{\p_0}\rightarrow 1$, $\hat{\q_0}\rightarrow 1$, and $\hat{\c_0}\rightarrow 1$, $\hat{\p}_{r+1}=\hat{\q}_{r+1}=\hat{\c}_{r+1}=0$, and let the non-fixed parts of $\hat{\p}\triangleq \hat{\p}(x,y)$, $\hat{\q}\triangleq \hat{\q}(x,y)$, and  $\hat{\c}\triangleq \hat{\c}(x,y)$ be the solutions of the following system
\begin{eqnarray}\label{eq:thmsflrdt2eq2}
   \frac{d \psi_{rd}(\p,\q,\c,x,y,s)}{d\p} =  0,\quad
   \frac{d \psi_{rd}(\p,\q,\c,x,y,s)}{d\q} =  0,\quad
   \frac{d \psi_{rd}(\p,\q,\c,x,y,s)}{d\c} =  0.
 \end{eqnarray}
 Then,
\begin{eqnarray}\label{eq:thmsflrdt2eq3}
    \lim_{n\rightarrow\infty} \frac{\mE_G  \psi_{rp}}{\sqrt{n}}
  & = &
\min_{x>0} \max_{y>0} \lim_{n\rightarrow\infty} \psi_{rd}(\hat{\p}(x,y),\hat{\q}(x,y),\hat{\c}(x,y),x,y,s) \qquad \mbox{(\bl{\textbf{strong sfl random duality}})},\nonumber \\
 \end{eqnarray}
where $\psi_{S,\infty}(\cdot)$ is as in (\ref{eq:fl4})-(\ref{eq:fl5}).
 \end{theorem}
\begin{proof}
Follows from the corresponding one proven in \cite{Stojnicflrdt23} in exactly he same way as Theorem 1 in \cite{Stojnicnegsphflrdt23}.
 \end{proof}

The above theorem holds generically for any given sets $\cX$ and $\cY$. The following corollary makes it fully operational for studying the statistical associative memory capacity which is of our interest here.
\begin{corollary}
\label{cor:cor1}  Assume the setup of Theorem \ref{thm:thmsflrdt1} with $\delta\in [0,1]$ and $\cX=\cX(\delta)$ and $\cY=\mS^m$ where $\cX(\delta)$ is as in (\ref{eq:ex0a15}) and $\mS^m$ is the unit $m$-dimensional Euclidean sphere. Set
\begin{align}\label{eq:thmsflrdt2eq1a0}
   \psi_{rp} & \triangleq - \max_{\x\in\cX} s \max_{\y\in\cY} \lp \y^TG\x  \rp
   \qquad  \mbox{(\bl{\textbf{random primal}})} \nonumber \\
   \psi_{rd}(\p,\q,\c,1,1,s) & \triangleq    \frac{1}{2}    \sum_{k=2}^{r+1}\Bigg(\Bigg.
   \p_{k-1}\q_{k-1}
   -\p_{k}\q_{k}
  \Bigg.\Bigg)
\c_k
  - \psi_{S,\infty}(0,\calX,\calY,\p,\q,\c,1,1,s) \quad \mbox{(\bl{\textbf{fl random dual}})}. \nonumber \\
 \end{align}
Let the non-fixed parts of $\hat{\p}$, $\hat{\q}$, and  $\hat{\c}$ be the solutions of the following system
\begin{eqnarray}\label{eq:thmsflrdt2eq2a0}
   \frac{d \psi_{rd}(\p,\q,\c,1,1,s)}{d\p} =  0,\quad
   \frac{d \psi_{rd}(\p,\q,\c,1,1,s)}{d\q} =  0,\quad
   \frac{d \psi_{rd}(\p,\q,\c,1,1,s)}{d\c} =  0.
 \end{eqnarray}
 Then,
\begin{eqnarray}\label{eq:thmsflrdt2eq3a0}
    \lim_{n\rightarrow\infty} \frac{\mE_G  \psi_{rp}}{\sqrt{n}}
  & = &
 \lim_{n\rightarrow\infty} \psi_{rd}(\hat{\p},\hat{\q},\hat{\c},1,1,s) \qquad \mbox{(\bl{\textbf{strong sfl random duality}})},\nonumber \\
 \end{eqnarray}
where $\psi_{S,\infty}(\cdot)$ is as in (\ref{eq:fl4})-(\ref{eq:fl5}).
 \end{corollary}
\begin{proof}
Follows trivially as an immediate consequence of Theorem \ref{thm:thmsflrdt1}, after choosing $f(\y)=0$ and recognizing that the specialized sets $\cX=\cX(\delta)$ and $\cY=\mS^m$ are such that each of their elements has unit Euclidean norm.
 \end{proof}

As  noted in \cite{Stojnicflrdt23,Stojnichopflrdt23}, trivial random primal concentrations enable various  probabilistic variants of (\ref{eq:thmsflrdt2eq3}) and (\ref{eq:thmsflrdt2eq3a0}) as well. We skip stating such trivialities though.

\section{Practical realization}
\label{sec:prac}

As usual, to ensure that Theorem \ref{thm:thmsflrdt1} and Corollary \ref{cor:cor1} are practically relevant, each of the underlying quantities needs to be evaluated. While trying to do so, two obstacles might appear as potentially unsurpassable: (i) A priori, it is not clear what the correct value for $r$ should be; and (ii) Neither of the sets $\cX$ and $\cY$ has a component-wise structural characterization which guarantees that the decoupling over both $\x$ and $\y$ is straightforward (if doable at all). It turns out, however, that each of these potential obstacles can be successfully surpassed.

 Taking the sets specializations, $\cX=\cX(\delta)$ and $\cY=\mS^m$, and relying on results of Corollary \ref{cor:cor1}, we start by noting that the key object of practical relevance is the following \emph{random dual}
\begin{align}\label{eq:prac1}
    \psi_{rd}(\p,\q,\c,1,1,s) & \triangleq    \frac{1}{2}    \sum_{k=2}^{r+1}\Bigg(\Bigg.
   \p_{k-1}\q_{k-1}
   -\p_{k}\q_{k}
  \Bigg.\Bigg)
\c_k
  - \psi_{S,\infty}(0,\cX,\cY,\p,\q,\c,1,1,s) \nonumber \\
  & \triangleq    \frac{1}{2}    \sum_{k=2}^{r+1}\Bigg(\Bigg.
   \p_{k-1}\q_{k-1}
   -\p_{k}\q_{k}
  \Bigg.\Bigg)
\c_k
  - \psi_{S,\infty}(0,\cX(\delta),\mS^m,\p,\q,\c,1,1,s) \nonumber \\
  & =   \frac{1}{2}    \sum_{k=2}^{r+1}\Bigg(\Bigg.
   \p_{k-1}\q_{k-1}
   -\p_{k}\q_{k}
  \Bigg.\Bigg)
\c_k
  - \frac{1}{n}\varphi(D^{(per)}(s),
  \c) - \frac{1}{n}\varphi(D^{(sph)}(s),\c), \nonumber \\
  \end{align}
where, based on (\ref{eq:fl4})-(\ref{eq:fl5}),
  \begin{eqnarray}\label{eq:prac2}
\varphi(D,\c) & \triangleq &
 \mE_{G,{\mathcal U}_{r+1}} \frac{1}{\c_r} \log
\lp \mE_{{\mathcal U}_{r}} \lp \dots \lp \mE_{{\mathcal U}_3}\lp\lp\mE_{{\mathcal U}_2} \lp
\lp    e^{D}   \rp^{\c_2}\rp\rp^{\frac{\c_3}{\c_2}}\rp\rp^{\frac{\c_4}{\c_3}} \dots \rp^{\frac{\c_{r}}{\c_{r-1}}}\rp, \nonumber \\
  \end{eqnarray}
and
\begin{eqnarray}\label{eq:prac3}
D^{(per)}(s) & = & \max_{\x\in\cX(\delta)} \lp   s\sqrt{n}      \lp\sum_{k=2}^{r+1}c_k\h^{(k)}\rp^T\x  \rp \nonumber \\
  D^{(sph)}(s) & \triangleq  &   s \max_{\y\in\mS^m}
\lp  \sqrt{n}  \y^T\lp\sum_{k=2}^{r+1}b_k\u^{(2,k)}\rp \rp.
 \end{eqnarray}

\noindent  \underline{\red{\textbf{\emph{(i) Handling $D^{(per)}(1)$:}}}}  One first easily finds
\begin{eqnarray}\label{eq:prac4}
D^{(per)}(1) & = & \max_{\x\in\cX(\delta)}   \lp \sqrt{n}      \lp\sum_{k=2}^{r+1}c_k\h^{(k)}\rp^T\x \rp =
\sqrt{n} \max_{\x\in\cX(\delta)}   \lp      \lp\sum_{k=2}^{r+1}c_k\h^{(k)}\rp^T\x \rp.
 \end{eqnarray}
Relying on (\ref{eq:ex0a15}), one then also observes that $\cX$ can be parameterized in the following way as well
\begin{eqnarray}
\cX(\delta) & \triangleq & \left \{\x | \quad \|\x-\bg^{(1)}\|_2^2=4\delta,\x\in S_H^n\right \} \nonumber \\
& = & \left \{\x | \quad \x=\frac{\1}{\sqrt{n}}-\z,\z\in \left \{0,\frac{2}{\sqrt{n}} \right \}^n, \sum_{i=1}^{n}\z_i=2\delta\sqrt{n}\right \},
 \label{eq:prac4a0a1}
\end{eqnarray}
where we recall that, as mentioned earlier, due to distributional symmetry, one can, without loss of generality, assume $\bg^{(1)}=\frac{\1}{\sqrt{n}}$. A combination of (\ref{eq:prac4}) and (\ref{eq:prac4a0a1}) gives
\begin{eqnarray}\label{eq:prac4a0a2}
D^{(per)}(1) & = &  -\sqrt{n} \min_{\z}   \lp     - \lp\sum_{k=2}^{r+1}c_k\h^{(k)}\rp^T\lp \frac{\1}{\sqrt{n}} -\z\rp \rp \nonumber \\
\mbox{subject to} & &  \z\in \left \{0,\frac{2}{\sqrt{n}} \right \}^n, \sum_{i=1}^{n}\z_i=2\delta\sqrt{n}.
 \end{eqnarray}
Writing the Lagrangian and utilizing the strong Lagrangian duality, we then obtain
\begin{eqnarray}\label{eq:prac4a0a3}
D^{(per)}(1) & =  &  -\sqrt{n} \min_{\z\in \left \{0,\frac{2}{\sqrt{n}} \right \}^n} \max_{\nu}
  \lp     - \lp\sum_{k=2}^{r+1}c_k\h^{(k)}\rp^T \lp\frac{\1}{\sqrt{n}}-\z\rp -\nu \sum_{i=1}^{n}\z_i+2\nu\delta\sqrt{n} \rp \nonumber \\
  & =  &  -\sqrt{n} \max_{\nu} \min_{\z\in \left \{0,\frac{2}{\sqrt{n}} \right \}^n}
  \lp     - \lp\sum_{k=2}^{r+1}c_k\h^{(k)}\rp^T \lp\frac{\1}{\sqrt{n}}-\z\rp -\nu \sum_{i=1}^{n}\z_i+2\nu\delta\sqrt{n} \rp.
 \end{eqnarray}
After the optimization over $\z$, we then also have
\begin{eqnarray}\label{eq:prac4a0a4}
D^{(per)}(1) & =  & - \max_{\nu} \lp 2\nu\delta n  + 2 \min\lp \sum_{k=2}^{r+1}c_k\h^{(k)} -\nu,0 \rp  - \lp\sum_{k=2}^{r+1}c_k\h^{(k)}\rp^T\1 \rp
\nonumber \\
  & =  & \min_{\nu} \lp  -2\nu\delta n  - 2 \min\lp \sum_{k=2}^{r+1}c_k\h^{(k)} -\nu,0 \rp  + \lp\sum_{k=2}^{r+1}c_k\h^{(k)}\rp^T\1 \rp \nonumber \\
   & =  & \min_{\nu} \lp  -2\nu\delta n  + \sum_{i=1}^nD^{(per)}_i(c_k) \rp,
 \end{eqnarray}
where
\begin{eqnarray}\label{eq:prac5}
D^{(per)}_i(c_k)=-2\min \lp\sum_{k=2}^{r+1}c_k\h_i^{(k)}-\nu,0\rp +\sum_{k=2}^{r+1}c_k\h_i^{(k)}.
\end{eqnarray}

\noindent \underline{\red{\textbf{\emph{(ii) Handling $D^{(sph)}(1)$:}}}}  We start by observing
\begin{eqnarray}\label{eq:prac7}
   D^{(sph)}(1) & \triangleq  &   \sqrt{n}  \max_{\y\in\mS^m}
\lp  \y^T\lp\sum_{k=2}^{r+1}b_k\u^{(2,k)}\rp \rp
=   \sqrt{n}   \left \| \sum_{k=2}^{r+1}b_k\u^{(2,k)} \right \|_2.
 \end{eqnarray}
The \emph{square root trick} introduced on numerous occasions in \cite{StojnicMoreSophHopBnds10,StojnicLiftStrSec13,StojnicGardSphErr13,StojnicGardSphNeg13}, turns out as particularly useful in handling
(\ref{eq:prac7}). To that end, we write
\begin{align}\label{eq:prac8}
   D^{(sph)} (1)
& =   \sqrt{n}   \left \| \sum_{k=2}^{r+1}b_k\u^{(2,k)}  \right \|_2
=  \sqrt{n}  \min_{\gamma} \lp \frac{\left \|\sum_{k=2}^{r+1}b_k\u^{(2,k)} \right \|_2^2}{4\gamma}+\gamma \rp \nonumber \\
 & =   \sqrt{n}  \min_{\gamma} \lp \frac{\sum_{i=1}^{m} \lp \sum_{k=2}^{r+1}b_k\u_i^{(2,k)} \rp ^2}{4\gamma}+\gamma \rp.
 \end{align}
We then introduce scaling $\gamma=\gamma_{sq}\sqrt{n}$ and rewrite  (\ref{eq:prac8}) as
\begin{eqnarray}\label{eq:prac9}
   D^{(sph)}(1)
  & =  & \sqrt{n}  \min_{\gamma_{sq}} \lp \frac{\sum_{i=1}^{m}  \lp\sum_{k=2}^{r+1}b_k\u_i^{(2,k)} \rp^2}{4\gamma_{sq}\sqrt{n}}+\gamma_{sq}\sqrt{n} \rp   \nonumber \\
  & = &  \min_{\gamma_{sq}} \lp \frac{\sum_{i=1}^{m} \lp\sum_{k=2}^{r+1}b_k\u_i^{(2,k)} \rp^2}{4\gamma_{sq}}+\gamma_{sq}n \rp \nonumber \\
  & =  &   \min_{\gamma_{sq}} \lp \sum_{i=1}^{m} D_i^{(sph)}(b_k)+\gamma_{sq}n \rp, \nonumber \\
 \end{eqnarray}
with
\begin{eqnarray}\label{eq:prac10}
   D_i^{(sph)}(b_k)= \frac{\lp  \sum_{k=2}^{r+1}b_k\u_i^{(2,k)} \rp^2}{4\gamma_{sq}}.
 \end{eqnarray}

Recalling on (\ref{eq:ex0a14}), we then have
 \begin{eqnarray}
\xi(\delta)
 & = &
 \lim_{n\rightarrow\infty}\frac{\mE_G \max_{\x\in\cX(\delta)}   \max_{\y\in\mS^m} \y^TG\x}{\sqrt{n}}
 =
    \lim_{n\rightarrow\infty} \frac{\mE_G  \psi_{rp}}{\sqrt{n}}
   =
 \lim_{n\rightarrow\infty} \psi_{rd}(\hat{\p},\hat{\q},\hat{\c},1,1,1).
  \label{eq:negprac11}
\end{eqnarray}
Keeping in mind (\ref{eq:prac1})-(\ref{eq:prac10}), we can then formulate the following theorem that ultimately enables fitting the associative memorization within the sfl RDT machinery.

\begin{theorem}
  \label{thme:negthmprac1}
  Assume the setup of Theorem \ref{thm:thmsflrdt1} and consider large $n$ linear regime with $\alpha=\lim_{n\rightarrow\infty} \frac{m}{n}$. Set
  \begin{eqnarray}\label{eq:thm2eq1}
\varphi(D,\c) & = &
 \mE_{G,{\mathcal U}_{r+1}} \frac{1}{\c_r} \log
\lp \mE_{{\mathcal U}_{r}} \lp \dots \lp \mE_{{\mathcal U}_3}\lp\lp\mE_{{\mathcal U}_2} \lp
\lp    e^{D}   \rp^{\c_2}\rp\rp^{\frac{\c_3}{\c_2}}\rp\rp^{\frac{\c_4}{\c_3}} \dots \rp^{\frac{\c_{r}}{\c_{r-1}}}\rp, \nonumber \\
  \end{eqnarray}
and
\begin{eqnarray}\label{eq:thm2eq2}
    \bar{\psi}_{rd}(\p,\q,\c,\gamma_{sq},\nu,\delta) & \triangleq &     - \frac{1}{2}    \sum_{k=2}^{r+1}\Bigg(\Bigg.
   \p_{k-1}\q_{k-1}
   -\p_{k}\q_{k}
  \Bigg.\Bigg)
\c_k
 -  2\nu\delta  +\varphi(D_1^{(per)}(1),\c)  + \gamma_{sq}  + \alpha\varphi(D_1^{(sph)}(1),\c), \nonumber \\
  \end{eqnarray}
where $D_1^{(per)}(1)$ and $D_1^{(sph)}(1)$ are as in (\ref{eq:prac5}) and (\ref{eq:prac10}), respectively. Let the ``fixed'' parts of $\hat{\p}$, $\hat{\q}$, and $\hat{\c}$ satisfy $\hat{\p}_1\rightarrow 1$, $\hat{\q}_1\rightarrow 1$, $\hat{\c}_1\rightarrow 1$, $\hat{\p}_{r+1}=\hat{\q}_{r+1}=\hat{\c}_{r+1}=0$, and let the ``non-fixed'' parts of $\hat{\p}_k$, $\hat{\q}_k$, and $\hat{\c}_k$ ($k\in\{2,3,\dots,r\}$) and $\hat{\gamma}_{sq}$ and $\hat{\nu}$ be the solutions of the following system of equations
  \begin{eqnarray}\label{eq:negthmprac1eq1}
   \frac{d \bar{\psi}_{rd}(\p,\q,\c,\gamma_{sq},\nu,\delta)}{d\p} =  0 \nonumber \\
   \frac{d \bar{\psi}_{rd}(\p,\q,\c,\gamma_{sq},\nu,\delta)}{d\q} =  0 \nonumber \\
   \frac{d \bar{\psi}_{rd}(\p,\q,\c,\gamma_{sq},\nu,\delta)}{d\c} =  0 \nonumber \\
   \frac{d \bar{\psi}_{rd}(\p,\q,\c,\gamma_{sq},\nu,\delta)}{d\gamma_{sq}} =  0\nonumber \\
   \frac{d \bar{\psi}_{rd}(\p,\q,\c,\gamma_{sq},\nu,\delta)}{d\nu} =  0,
 \end{eqnarray}
 and, consequently, let
\begin{eqnarray}\label{eq:prac17}
c_k(\hat{\p},\hat{\q})  & = & \sqrt{\hat{\q}_{k-1}-\hat{\q}_k} \nonumber \\
b_k(\hat{\p},\hat{\q})  & = & \sqrt{\hat{\p}_{k-1}-\hat{\p}_k}.
 \end{eqnarray}
 Then
 \begin{equation}
\xi(\delta) =   \bar{\psi}_{rd}(\hat{\p},\hat{\q},\hat{\c},\hat{\gamma}_{sq},\hat{\nu},\delta).
  \label{eq:negthmprac1eq2}
\end{equation}
\end{theorem}
\begin{proof}
Follows from the above discussion, Theorem \ref{thm:thmsflrdt1}, Corollary \ref{cor:cor1}, and the sfl RDT machinery presented in \cite{Stojnicnflgscompyx23,Stojnicsflgscompyx23,Stojnicflrdt23,Stojnichopflrdt23}.
\end{proof}

\subsection{Numerical evaluations -- AGS basin}
\label{sec:nuemricalags}

As we have stated earlier, to have the results of Theorem \ref{thme:negthmprac1} become practically operational  one needs to conduct all the underlying numerical evaluations. Since all technical ingredients needed for the  evaluations are already present in the theorem itself, we start the evaluations with $r=1$ and proceed by incrementally increasing $r$. This allows us to have a systematic view on how the entire lifting machinery is progressing. Also, as mentioned earlier, several analytical closed form results can be obtained as well. They turn out to be critically important for simplification of the entire evaluation process and we state them below as well.

\subsubsection{$r=1$ -- first level of lifting}
\label{sec:firstlev}

For the first level, we have $r=1$ and $\hat{\p}_1\rightarrow 1$ and $\hat{\q}_1\rightarrow 1$ which, together with $\hat{\p}_{r+1}=\hat{\p}_{2}=\hat{\q}_{r+1}=\hat{\q}_{2}=0$, and $\hat{\c}_{2}\rightarrow 0$, gives
\begin{align}\label{eq:negprac19}
    \bar{\psi}_{rd}^{(1)}(\hat{\p},\hat{\q},\hat{\c},\gamma_{sq},\nu,\delta)   & =  - \frac{1}{2}
\hat{\c}_2
- 2\nu\delta + \frac{1}{\hat{\c}_2}\log\lp \mE_{{\mathcal U}_2} e^{\hat{\c}_2  \lp  -2\min  \lp \sqrt{1-0}\h_1^{(2)}-\nu,0 \rp +\sqrt{1-0}\h_1^{(2)} \rp}\rp
 \nonumber \\
&  \quad +\gamma_{sq}
+ \alpha\frac{1}{\hat{\c}_2}\log\lp \mE_{{\mathcal U}_2} e^{\hat{\c}_2\frac{\lp  \sqrt{1-0}\u_1^{(2,2)} \rp^2}{4\gamma_{sq}}}\rp \nonumber \\
& \rightarrow
- 2\nu\delta + \frac{1}{\hat{\c}_2}\log\lp 1+ \mE_{{\mathcal U}_2} \hat{\c}_2  \lp  -2\min  \lp \sqrt{1-0}\h_1^{(2)}-\nu,0 \rp +\sqrt{1-0}\h_1^{(2)} \rp \rp
 \nonumber \\
&  \quad +\gamma_{sq}
+ \alpha\frac{1}{\hat{\c}_2}\log\lp 1 + \mE_{{\mathcal U}_2} \hat{\c}_2\frac{\lp  \sqrt{1-0}\u_1^{(2,2)} \rp^2}{4\gamma_{sq}}\rp \nonumber \\
& \rightarrow
- 2\nu\delta + \frac{1}{\hat{\c}_2}\log\lp 1+ \hat{\c}_2  \mE_{{\mathcal U}_2}  \lp  -2\min  \lp \h_1^{(2)}-\nu,0 \rp +\h_1^{(2)} \rp \rp
 \nonumber \\
&  \quad +\gamma_{sq}
+ \alpha\frac{1}{\hat{\c}_2}\log\lp 1 + \frac{\hat{\c}_2}{4\gamma_{sq}}\rp \nonumber \\
& \rightarrow
- 2\nu\delta + \mE_{{\mathcal U}_2}  \lp  -2\min  \lp \h_1^{(2)}-\nu,0 \rp +\h_1^{(2)} \rp
 \nonumber \\
&  \quad +\gamma_{sq}
+   \frac{\alpha}{4\gamma_{sq}}.
  \end{align}
From   $\frac{ \bar{\psi}_{rd}(\hat{\p},\hat{\q},\hat{\c},\gamma_{sq},\nu,\delta) }{d\gamma_{sq}}=0$, one easily finds
$\hat{\gamma}_{sq}=\frac{\sqrt{\alpha}}{2}$, which then gives
\begin{eqnarray}\label{eq:negprac19a0}
    \bar{\psi}_{rd}^{(1)}(\hat{\p},\hat{\q},\hat{\c},\hat{\gamma}_{sq},\nu,\delta)   &  \rightarrow &
- 2\nu\delta + \mE_{{\mathcal U}_2}  \lp  -2\min  \lp \h_1^{(2)}-\nu,0 \rp +\h_1^{(2)} \rp
+ \sqrt{\alpha} \nonumber \\
&  \rightarrow &
 f_1(\nu) + \sqrt{\alpha},
  \end{eqnarray}
where
\begin{eqnarray}\label{eq:negprac19a1}
f_1(\nu) \triangleq \mE_{{\mathcal U}_2}  \lp  2\max  \lp -\h_1^{(2)}+\nu,0 \rp +\h_1^{(2)} \rp - 2\nu\delta
=\mE_{{\mathcal U}_2}  \lp  2\max  \lp -\h_1^{(2)}+\nu,0 \rp \rp - 2\nu\delta .
  \end{eqnarray}
After computing the above expectation, one obtains
\begin{eqnarray}\label{eq:negprac19a2}
f_1(\nu) & = & \mE_{{\mathcal U}_2}  \lp  2\max  \lp -\h_1^{(2)}+\nu,0 \rp \rp - 2\nu\delta  \nonumber \\
& = & 2\lp \frac{1}{\sqrt{2\pi}} e^{-\nu^2/2}+\frac{\nu}{2}\erfc\lp - \frac{\nu}{\sqrt{2}}\rp\rp-2\nu\delta.
  \end{eqnarray}
One then further has
\begin{eqnarray}\label{eq:negprac19a3}
\frac{df_1(\nu)}{d\nu}
& = & 2\lp \frac{1}{2}\erfc\lp -\frac{\nu}{\sqrt{2}}\rp \rp-2\delta  = \erfc \lp -\frac{\nu}{\sqrt{2}}\rp -2\delta
 =1-2\delta - \erf \lp -\frac{\nu}{\sqrt{2}}\rp .
  \end{eqnarray}
From   $\frac{ \bar{\psi}_{rd}(\hat{\p},\hat{\q},\hat{\c},\gamma_{sq},\nu,\delta) }{d\nu}=\frac{df_1(\nu)}{d\nu}
=0$, one also finds
$\hat{\nu}=-\sqrt{2} \erfinv \lp 1-2\delta\rp$, which then gives
\begin{eqnarray}\label{eq:negprac19a4}
f_1(\hat{\nu})
& = & 2\lp \frac{1}{\sqrt{2\pi}} e^{-\frac{\hat{\nu}^2}{2}}+\frac{\hat{\nu}}{2}\erfc\lp - \frac{\hat{\nu}}{\sqrt{2}}\rp\rp-2\hat{\nu}\delta
=
 \frac{2}{\sqrt{2\pi}} e^{-\frac{\hat{\nu}^2}{2}}+\hat{\nu}\erfc\lp - \frac{\hat{\nu}}{\sqrt{2}}\rp - 2\hat{\nu}\delta \nonumber \\
& = &
 \frac{2}{\sqrt{2\pi}} e^{-\frac{\hat{\nu}^2}{2}}+\hat{\nu} -\hat{\nu}\erf\lp - \frac{\hat{\nu}}{\sqrt{2}}\rp - 2\hat{\nu}\delta
 =
  \frac{2}{\sqrt{2\pi}} e^{-\frac{\hat{\nu}^2}{2}}+\hat{\nu} -\hat{\nu}\erf\lp \erfinv \lp 1-2\delta\rp\rp - 2\hat{\nu}\delta \nonumber \\
 & = &
  \frac{2}{\sqrt{2\pi}} e^{-\frac{\hat{\nu}^2}{2}}+\hat{\nu} -\hat{\nu} \lp 1-2\delta \rp - 2\hat{\nu}\delta
  =  \frac{2}{\sqrt{2\pi}} e^{-\frac{\hat{\nu}^2}{2}}.
  \end{eqnarray}
A combination of (\ref{eq:ex0a12}), (\ref{eq:negthmprac1eq2}), (\ref{eq:negprac19a0}), and (\ref{eq:negprac19a4}) gives
\begin{eqnarray}
\xi_1(\delta) & = &
-\lp 1-2\delta\rp^2 -
\xi(\delta)^2
  =  -\lp 1-2\delta\rp^2 -
 \lp \bar{\psi}_{rd}(\hat{\p},\hat{\q},\hat{\c},\hat{\gamma}_{sq},\hat{\nu},\delta)\rp^2  \nonumber \\
&  = &  -\lp 1-2\delta\rp^2 -
 \lp f_1(\hat{\nu}) +\sqrt{\alpha} \rp^2
 = -\lp 1-2\delta\rp^2 -
 \lp  \frac{2}{\sqrt{2\pi}} e^{-\frac{\hat{\nu}^2}{2}} +\sqrt{\alpha} \rp^2 \nonumber \\
& = & -\lp 1-2\delta\rp^2 -
 \lp  \frac{2}{\sqrt{2\pi}} e^{-\lp \erfinv\lp 1-2\delta \rp\rp^2} +\sqrt{\alpha} \rp^2
 \label{eq:negprac19a6}
\end{eqnarray}
Focusing on AGS basin related part of (\ref{eq:ex0a16}), one observes that the stationary points of $\xi_1(\delta)$ are of critical importance for capacity evaluation. To that end, we continue further  by writing
\begin{eqnarray}\label{eq:negprac19a7}
 \frac{d \hat{\nu}}{d\delta}=\sqrt{2\pi}e^{(\erfinv\lp1-2\delta\rp)^2},
\end{eqnarray}
  and
\begin{eqnarray}\label{eq:negprac19a8}
\frac{df_1(\hat{\nu})
}{d\delta}=  -\hat{\nu}\frac{2}{\sqrt{2\pi}} e^{-\frac{\hat{\nu}^2}{2}} \frac{d \hat{\nu}}{d\delta}
=
-\hat{\nu}\frac{2}{\sqrt{2\pi}} e^{-\frac{\hat{\nu}^2}{2}}
\sqrt{2\pi}e^{(\erfinv\lp1-2\delta\rp)^2} =-2\hat{\nu}=-2\sqrt{2} \erfinv \lp 1-2\delta\rp.
\end{eqnarray}
Combining (\ref{eq:negprac19a6}) and (\ref{eq:negprac19a8}), we find
\begin{eqnarray}
\frac{d\xi_1(\delta)}{d\delta}
& = & 4\lp 1-2\delta\rp - 2
 \lp f_1(\hat{\nu}) +\sqrt{\alpha} \rp  \frac{df_1(\hat{\nu})}{d\delta} \nonumber \\
& = & 4\lp 1-2\delta\rp - 4 \sqrt{2} \erfinv \lp 1-2\delta\rp \lp  \frac{2}{\sqrt{2\pi}} e^{-\lp \erfinv\lp 1-2\delta \rp\rp^2} +\sqrt{\alpha} \rp.
 \label{eq:negprac19a9}
\end{eqnarray}
From the capacity point of view, the most critical scenario happens when the stationary points of $\xi_1(\delta)$ merge to become an infliction point. To analytically determine the infliction point, the second derivative is needed as well. To that end, we first have
\begin{eqnarray}
 \frac{d^2f_1(\hat{\nu})}{d\delta^2}=-2 \frac{d \hat{\nu}}{d\delta}=-2\sqrt{2\pi}e^{(\erfinv\lp1-2\delta\rp)^2},
 \label{eq:negprac19a10}
\end{eqnarray}
and then
\begin{eqnarray}
\frac{d^2\xi_1(\delta)}{d\delta^2}
& = & -8 -2 \lp\frac{df_1(\hat{\nu})}{d\delta}\rp^2 - 2
 \lp f_1(\hat{\nu}) +\sqrt{\alpha} \rp  \frac{d^2f_1(\hat{\nu})}{d\delta^2} \nonumber \\
& = & -8 -8 \hat{\nu}^2 + 4\sqrt{2\pi}
 \lp f_1(\hat{\nu}) +\sqrt{\alpha} \rp  e^{(\erfinv\lp1-2\delta\rp)^2} \nonumber \\
& = & -8 -8 \hat{\nu}^2 + 4\sqrt{2\pi}
  f_1(\hat{\nu})  e^{(\erfinv\lp1-2\delta\rp)^2} + 4\sqrt{2\pi}\sqrt{\alpha}  e^{(\erfinv\lp1-2\delta\rp)^2} \nonumber \\
 & = & -8 -8 \hat{\nu}^2 + 4\sqrt{2\pi}
   \frac{2}{\sqrt{2\pi}} e^{-(\erfinv\lp1-2\delta\rp)^2}  e^{(\erfinv\lp1-2\delta\rp)^2} + 4\sqrt{2\pi}\sqrt{\alpha}  e^{(\erfinv\lp1-2\delta\rp)^2} \nonumber \\
 & = & -8 \hat{\nu}^2   + 4\sqrt{2\pi}\sqrt{\alpha}  e^{(\erfinv\lp1-2\delta\rp)^2} \nonumber \\
 & = & -16\erfinv(1-2\delta)^2   + 4\sqrt{2\pi}\sqrt{\alpha}  e^{(\erfinv\lp1-2\delta\rp)^2}.
 \label{eq:negprac19a11}
\end{eqnarray}
To determine the infliction point one needs
\begin{eqnarray}
\frac{d\xi_1(\delta)}{d\delta}
 =  \frac{d^2\xi_1(\delta)}{d\delta^2}
 =   0.
 \label{eq:negprac19a12}
\end{eqnarray}
From (\ref{eq:negprac19a11}) and (\ref{eq:negprac19a12}) we then have
\begin{eqnarray}
 \sqrt{\alpha}=\frac{4}{\sqrt{2\pi}}\erfinv(1-2\delta)^2e^{-(\erfinv\lp1-2\delta\rp)^2}.
 \label{eq:negprac19a13}
\end{eqnarray}
Plugging $\sqrt{\alpha}$ from (\ref{eq:negprac19a13}) back into (\ref{eq:negprac19a9}) and keeping in mind (\ref{eq:negprac19a12}), we obtain the following critical equation
\begin{eqnarray}
 4\lp 1-2\delta\rp - 4 \sqrt{2} \erfinv \lp 1-2\delta\rp \lp  \frac{2}{\sqrt{2\pi}} e^{-\lp \erfinv\lp 1-2\delta \rp\rp^2} +\frac{4}{\sqrt{2\pi}}\erfinv(1-2\delta)^2e^{-(\erfinv\lp1-2\delta\rp)^2} \rp
=0.
 \label{eq:negprac19a14}
\end{eqnarray}
After denoting by $\hat{\delta}$ the solution of (\ref{eq:negprac19a14}), we then from (\ref{eq:negprac19a13}) finally have
 \begin{equation}\label{eq:negprac21}
\hspace{-0.5in}(\mbox{\bl{\textbf{first level:}}}) \qquad \qquad   \alpha_c^{(AGS,1)}=\lp\frac{4}{\sqrt{2\pi}}\erfinv(1-2\hat{\delta})^2e^{-(\erfinv\lp1-2\hat{\delta}\rp)^2}\rp^2
\approx  \bl{\mathbf{0.137905566}}.
  \end{equation}
We show in Figure \ref{fig:fig1} $\xi_{tot}(\delta)$ defined as
\begin{equation}\label{eq:negprac21a0}
  \xi_{tot}(\delta)=\xi_1(\delta)-\xi_1(0)=-(1-2\delta)^2-\xi(\delta)^2+1+\alpha.
\end{equation}
It is not that difficult to see that $ \xi_{tot}(\delta)$ is the shifted version of $\xi_1(\delta)$ which accounts for the difference with respect to the energy of the aimed memorized pattern.
\begin{figure}[h]
\centering
\centerline{\includegraphics[width=1\linewidth]{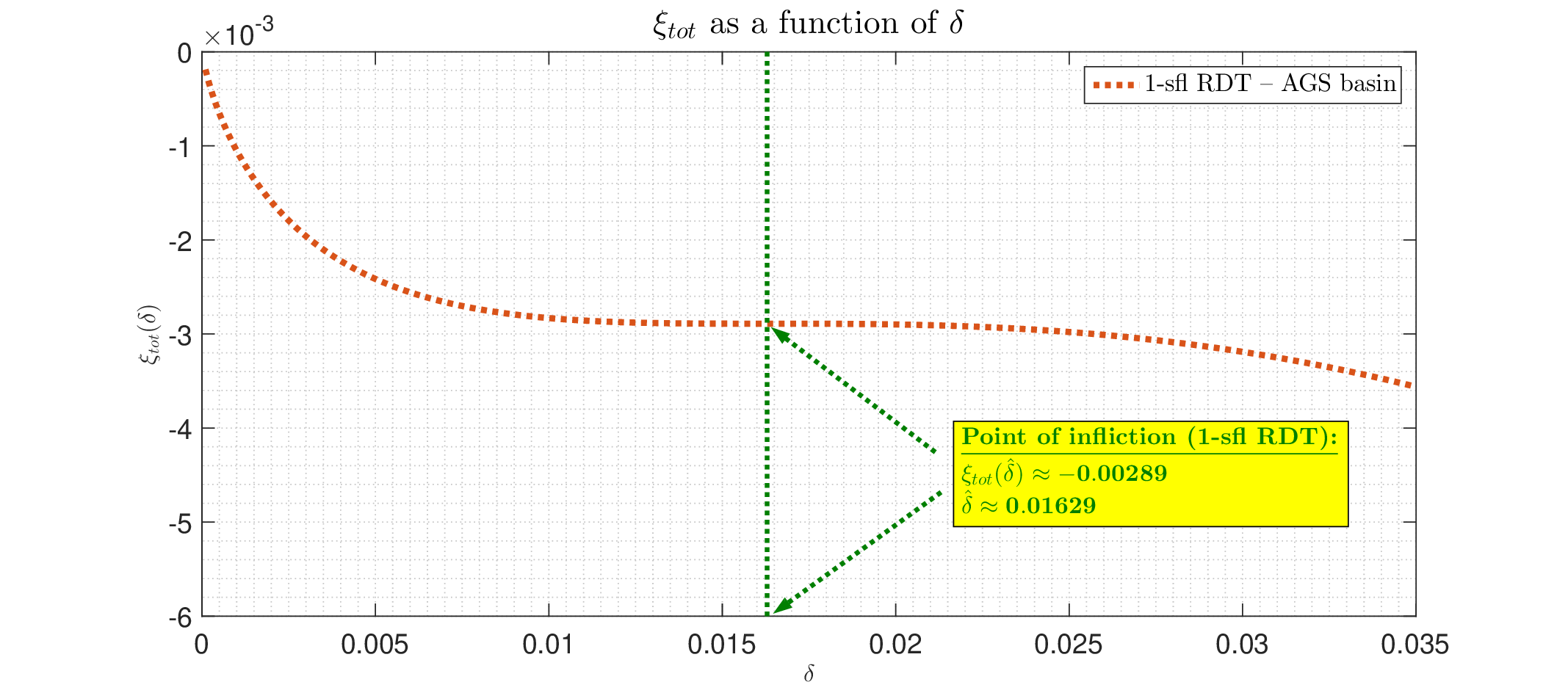}}
\caption{$\xi_{tot}$ as a function of $\delta$; $\alpha_c^{(AGS,1)}
\approx  \bl{\mathbf{0.137905566}}$ -- maximum $\alpha$ such that the infliction point still exists on the first level of lifting}
\label{fig:fig1}
\end{figure}

The above analysis proceeded rather smoothly since the underlying analytical calculations can be explicitly performed. In general, however, that might not be the case. In such situations, it is useful to observe that the associative memory capacity in the AGS basin sense, based on (\ref{eq:ex0a16}), can alternatively be obtained as
 \begin{eqnarray}
\alpha & = &    \lim_{n\rightarrow \infty} \frac{m}{n}  \nonumber \\
\alpha_c^{(AGS)} & \triangleq & \max \left \{\alpha |\hspace{.08in} \exists \delta\in \lp0,\frac{1}{2}\rp, \frac{d\xi_1(\delta)}{d\delta}=0\right \}.
  \label{eq:altcap1}
\end{eqnarray}
Combining (\ref{eq:negprac19a9}) and (\ref{eq:altcap1}), one alternative finds
\begin{eqnarray}
 \alpha_c^{(AGS,1)} & = & \lp \max_{\delta\in \lp 0,\frac{1}{2}\rp }\frac{1-2\delta}{\sqrt{2} \erfinv \lp 1-2\delta\rp} -   \frac{2}{\sqrt{2\pi}} e^{-\lp \erfinv\lp 1-2\delta \rp\rp^2}\rp^2.
 \label{eq:altcap2}
\end{eqnarray}
It is not that difficult to check that $\hat{\delta}$ that satisfies (\ref{eq:negprac19a14}) is the solution to the above optimization. The above alternative view is visualized in Figure \ref{fig:fig1a}. As figure indicates,  $\alpha>\alpha_c^{(AGS,1)}$ would not be able to touch the right hand side of (\ref{eq:altcap2}), i.e., the blue curve in the figure.
\begin{figure}[h]
\centering
\centerline{\includegraphics[width=1\linewidth]{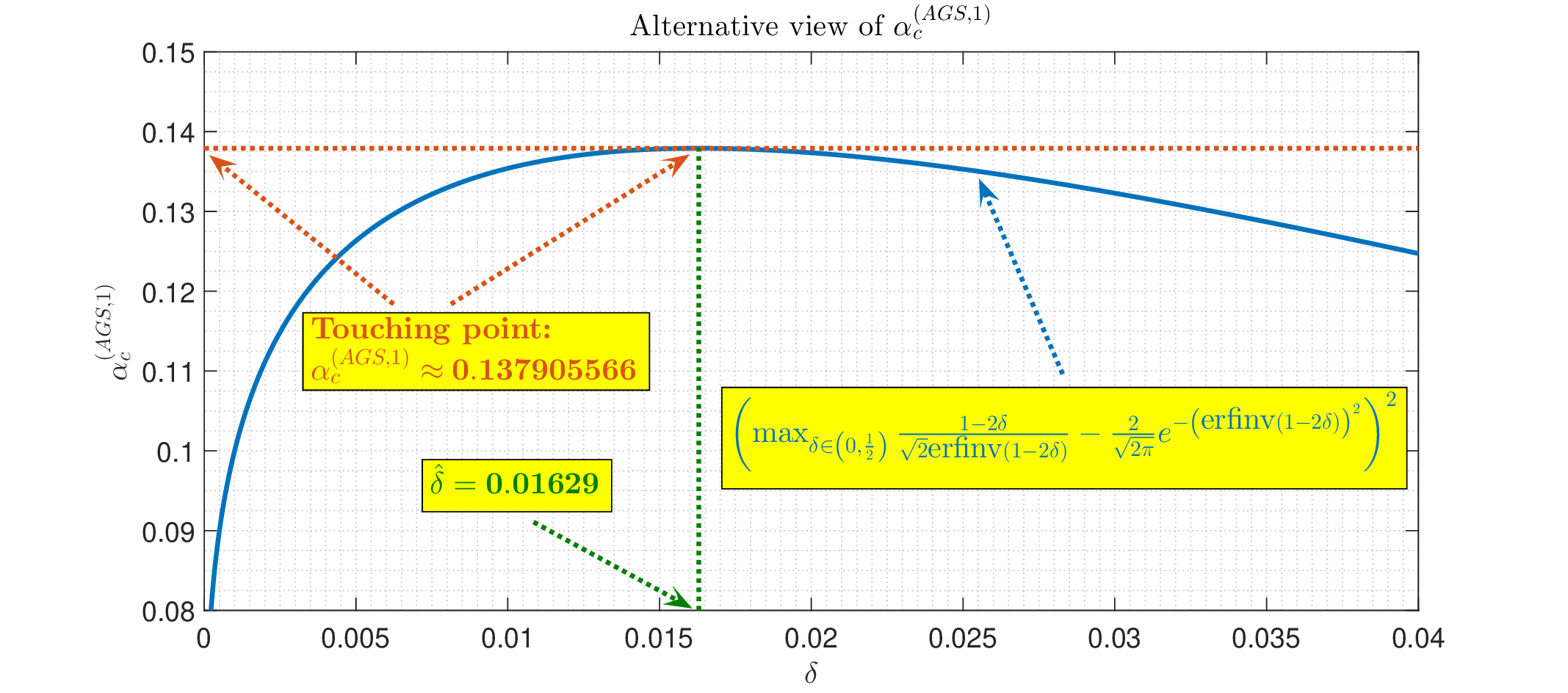}}
\caption{Alternative view of $\alpha_c^{(AGS,1)}
\approx  \bl{\mathbf{0.137905566}}$ -- maximum $\alpha$ such that the infliction point still exists on the first level of lifting}
\label{fig:fig1a}
\end{figure}

\subsubsection{$r=2$ -- second level of lifting}
\label{sec:secondlev}

The setup presented above can be utilized for the second level as well. We now, however, have $r=2$ and, similarly to what we had on the first level, $\hat{\p}_1\rightarrow 1$ and $\hat{\q}_1\rightarrow 1$. On the other hand, $\hat{\c}_{2}\neq 0$, $\p_2\neq0$, and $\q_2\neq0$ which, together with $\hat{\p}_{r+1}=\hat{\p}_{3}=\hat{\q}_{r+1}=\hat{\q}_{3}=0$ allows to write, analogously to (\ref{eq:negprac19}),
\begin{eqnarray}\label{eq:negprac24}
    \bar{\psi}_{rd}^{(2)}(\p,\q,\c,\gamma_{sq},\nu,\delta)
    & = &  -\frac{1}{2}
(1-\p_2\q_2)\c_2
-  2\nu\delta  \nonumber \\
& & + \frac{1}{\c_2}\mE_{{\mathcal U}_3}\log\lp \mE_{{\mathcal U}_2} e^{\c_2  \lp  -2\min  \lp \sqrt{1-\q_2}\h_1^{(2)}+\sqrt{\q_2}\h_1^{(3)}-\nu,0 \rp +\sqrt{1-\q_2}\h_1^{(2)}+\sqrt{\q_2}\h_1^{(3)} \rp }\rp \nonumber \\
& &   + \gamma_{sq}
 +\alpha\frac{1}{\c_2}\mE_{{\mathcal U}_3} \log\lp \mE_{{\mathcal U}_2} e^{\c_2\frac{\lp\sqrt{1-\p_2}\u_1^{(2,2)}+\sqrt{\p_2}\u_1^{(2,3)}\rp^2}{4\gamma_{sq}}}\rp \nonumber \\
& = & - \frac{1}{2}
(1-\p_2\q_2)\c_2
-  2\nu\delta  \nonumber \\
& & + \frac{1}{\c_2}\mE_{{\mathcal U}_3}\log\lp \mE_{{\mathcal U}_2} e^{\c_2  \lp  -2\min  \lp \sqrt{1-\q_2}\h_1^{(2)}+\sqrt{\q_2}\h_1^{(3)}-\nu,0 \rp +\sqrt{1-\q_2}\h_1^{(2)}+\sqrt{\q_2}\h_1^{(3)} \rp }\rp \nonumber \\
 & &  +  \gamma_{sq}
+ \alpha \Bigg(\Bigg. -\frac{1}{2\c_2} \log \lp \frac{2\gamma_{sq}-\c_2(1-\p_2)}{2\gamma_{sq}} \rp  +  \frac{\p_2}{2(2\gamma_{sq}-\c_2(1-\p_2))}   \Bigg.\Bigg).\nonumber \\
    \end{eqnarray}


After solving the remaining integrals, we also have
\begin{eqnarray}\label{eq:negprac24a2}
\hat{b} & = &  \frac{\sqrt{\q_2}\h_1^{(3)}-\nu}{\sqrt{1-\q_2}}    \nonumber \\
\hat{a} & = & \c_2\sqrt{1-\q_2}
\nonumber \\
f_{(zu)}^{(2,f)}& = & e^{\c_2\sqrt{q}\h_1^{(3)}}\lp\frac{1}{2} e^{\frac{\hat{a}^2}{2}}\lp\erf\lp\frac{\hat{a} + \hat{b}}{\sqrt{2}}\rp + 1\rp\rp
\nonumber \\
f_{(zd)}^{(2,f)}& = & e^{-\c_2\sqrt{q}\h_1^{(3)}+2\c_2\nu}\lp\frac{1}{2}e^{\frac{\hat{a}^2}{2}}\lp\erf\lp\frac{\hat{a} - \hat{b}}{\sqrt{2}}\rp + 1\rp\rp
\nonumber \\
f_{(zt)}^{(2,f)}& = & f_{(zd)}^{(2,f)}+f_{(zu)}^{(2,f)}.
   \end{eqnarray}
and
\begin{eqnarray}\label{eq:negprac24a3}
\mE_{{\mathcal U}_3}\log\lp \mE_{{\mathcal U}_2} e^{\c_2  \lp  -2\min  \lp \sqrt{1-\q_2}\h_1^{(2)}+\sqrt{\q_2}\h_1^{(3)}-\nu,0 \rp +\sqrt{1-\q_2}\h_1^{(2)}+\sqrt{\q_2}\h_1^{(3)} \rp }\rp
=   \mE_{{\mathcal U}_3} \log\lp f_{(zt)}^{(2,f)} \rp.
    \end{eqnarray}
The \emph{five} derivatives with respect to $\p_2$, $\q_2$, $\c_2$, $\gamma_{sq}$,, and $\nu$ need to be computed as well. This is done next.

\vspace{.1in}
\noindent \underline{\red{\textbf{(i) $\p_2$ -- derivative:}}} We start with the following
\begin{eqnarray}\label{eq:2levder1}
   \frac{d\bar{\psi}_{rd}^{(2)}(\p,\q,\c,\gamma_{sq},\nu,\delta) }{d\p_2}
 & = &  \frac{1}{2}
\q_2\c_2
+\alpha\Bigg(\Bigg.  -\frac{1}{2((2\gamma_{sq}-\c_2(1-\p_2)))}+\frac{1}{2((2\gamma_{sq}-\c_2(1-\p_2)))}\nonumber \\
& & -\frac{\p_2}{2(2\gamma_{sq}-\c_2(1-\p_2))^2}\c_2
  \Bigg.\Bigg)
 \nonumber \\
 & = &  \frac{1}{2}
\q_2\c_2
-\alpha \frac{\p_2}{2(2\gamma_{sq}^{(p)}-\c_2(1-\p_2))^2}\c_2 \nonumber \\
 & = &  \c_2\lp \frac{1}{2}
\q_2
-\alpha\frac{\p_2}{2(2\gamma_{sq}-\c_2(1-\p_2))^2}\rp.
     \end{eqnarray}

\vspace{.1in}
\noindent \underline{\red{\textbf{(ii) $\q_2$ -- derivative:}}} As above, we start with
\begin{equation}\label{eq:2levder2}
   \frac{d\bar{\psi}_{rd}^{(2)}(\p,\q,\c,\gamma_{sq},\nu,\delta) }{d\q_2}
  =   \frac{1}{2}
\p_2\c_2
+ \frac{1}{\c_2}\frac{d\lp \mE_{{\mathcal U}_3} \log\lp f_{(zt)}^{(2,f)} \rp \rp}{d\q_2} =  \frac{1}{2}
\p_2\c_2
+ \frac{1}{\c_2}  \mE_{{\mathcal U}_3} \lp \frac{1}{f_{(zt)}^{(2,f)}} \frac{d\lp f_{(zt)}^{(2,f)}\rp}{d\q_2} \rp.
     \end{equation}
From (\ref{eq:negprac24a2}), we further have
\begin{eqnarray}\label{eq:2levder3}
 \frac{df_{(zt)}^{(2,f)}}{d\q_2}& = &  \frac{df_{(zd)}^{(2,f)}}{d\q_2}+ \frac{df_{(zu)}^{(2,f)}}{d\q_2}.
   \end{eqnarray}
We first find
\begin{eqnarray}\label{eq:2levder4amin1}
\frac{d\hat{b}}{d\q_2} & = & \frac{\h_1^{(3)}}{2\sqrt{\q_2}\sqrt{1-\q_2}}-\frac{\nu}{2\sqrt{1-\q_2}^3} \nonumber \\
\frac{d\hat{a}}{d\q_2} & = & -\frac{\c_2}{2\sqrt{1-\q_2}}.
\end{eqnarray}
Then we obtain
\begin{eqnarray}\label{eq:2levder4}
 \frac{df_{(zu)}^{(2,f)}}{d\q_2} = f_{(\q,u)}^{(1)}+f_{(\q,u)}^{(2)},
   \end{eqnarray}
   where
\begin{eqnarray}\label{eq:2levder4a0}
 f_{(\q,u)}^{(1)} & = &  \lp\frac{\c_2}{2\sqrt{\q_2}\h_1^{(3)}}+\hat{a} \frac{d\hat{a}}{d\q_2} \rp \lp \frac{1}{2} e^{\c_2\sqrt{\q_2}\h_1^{(3)}+\hat{a}^2/2} \lp\erf \lp \frac{(\hat{a} + \hat{b})}{\sqrt{2}}\rp + 1\rp\rp \nonumber \\
 f_{(\q,u)}^{(2)} & = &    \frac{1}{2}e^{\c_2\sqrt{q}\h_1^{(3)}+\frac{\hat{a}^2}{2}}\lp \frac{2}{\sqrt{\pi}} e^{-\frac{(\hat{a} + \hat{b})^2}{2}}\lp\frac{d\hat{a}}{d\q_2}+\frac{d\hat{b}}{d\q_2}\rp\frac{1}{\sqrt{2}}\rp.
    \end{eqnarray}
One then analogously finds
 \begin{eqnarray}\label{eq:2levder4a1}
 \frac{df_{(zd)}^{(2,f)}}{d\q_2} = f_{(\q,d)}^{(1)}+f_{(\q,d)}^{(2)},
   \end{eqnarray}
   where
\begin{eqnarray}\label{eq:2levder4a2}
 f_{(\q,d)}^{(1)} & = &  \lp -\frac{\c_2}{2\sqrt{\q_2}\h_1^{(3)}}+\hat{a} \frac{d\hat{a}}{d\q_2} \rp \lp \frac{1}{2} e^{-\c_2\sqrt{\q_2}\h_1^{(3)}+2\c_2\nu+\hat{a}^2/2} \lp\erf \lp \frac{(\hat{a} - \hat{b})}{\sqrt{2}}\rp + 1\rp\rp \nonumber \\
 f_{(\q,d)}^{(2)} & = &    \frac{1}{2}e^{-\c_2\sqrt{q}\h_1^{(3)}+2\c_2\nu+\frac{\hat{a}^2}{2}}\lp \frac{2}{\sqrt{\pi}} e^{-\frac{(\hat{a} - \hat{b})^2}{2}}\lp\frac{d\hat{a}}{d\q_2}-\frac{d\hat{b}}{d\q_2}\rp\frac{1}{\sqrt{2}}\rp.
    \end{eqnarray}
 A combination of (\ref{eq:negprac24a2}), (\ref{eq:2levder2})-(\ref{eq:2levder4a2}) is then sufficient to determine $\q_2$--derivative.

\vspace{.1in}

\noindent \underline{\red{\textbf{(iii) $\c_2$ -- derivative:}}} We again start by writing
\begin{eqnarray}\label{eq:2levder12}
   \frac{d\bar{\psi}_{rd}^{(2)}(\p,\q,\c,\gamma_{sq},\nu,\delta) }{d\c_2}
 & = &  -\frac{1}{2}(1-
\p_2\q_2)
- \frac{1}{\c_2^2}  \mE_{{\mathcal U}_3} \log\lp f_{(zt)}^{(2,f)} \rp
+ \frac{1}{\c_2}  \mE_{{\mathcal U}_3} \lp \frac{1}{f_{(zt)}^{(2,f)}} \frac{d\lp f_{(zt)}^{(2,f)}\rp}{d\c_2} \rp \nonumber \\
& &
 +\alpha \Bigg (\Bigg. \frac{1}{2\c_2^2} \log \lp \frac{2\gamma_{sq}-\c_2(1-\p_2)}{2\gamma_{sq}} \rp
 +\frac{1-\p_2}{2\c_2(2\gamma_{sq}-\c_2(1-\p_2))}
 \nonumber \\
 & &
    +  \frac{\p_2(1-\p_2)}{2(2\gamma_{sq}-\c_2(1-\p_2))^2} \Bigg.\Bigg ).
      \end{eqnarray}
From (\ref{eq:negprac24a2}), we further have
\begin{eqnarray}\label{eq:2levder13}
 \frac{df_{(zt)}^{(2,f)}}{d\c_2}& = &  \frac{df_{(zd)}^{(2,f)}}{d\c_2}+ \frac{df_{(zu)}^{(2,f)}}{d\c_2}.
   \end{eqnarray}
Following what we presented above, we first find
\begin{eqnarray}\label{eq:2levder14amin1}
\frac{d\hat{b}}{d\c_2} & = & 0 \nonumber \\
\frac{d\hat{a}}{d\c_2} & = & \sqrt{1-\q_2}.
\end{eqnarray}
Then we obtain
\begin{eqnarray}\label{eq:2levder14}
 \frac{df_{(zu)}^{(2,f)}}{d\c_2} = f_{(\c,u)}^{(1)}+f_{(\c,u)}^{(2)},
   \end{eqnarray}
   where
\begin{eqnarray}\label{eq:2levder14a0}
 f_{(\c,u)}^{(1)} & = &  \lp \sqrt{\q_2}\h_1^{(3)}+\hat{a} \frac{d\hat{a}}{d\c_2} \rp \lp \frac{1}{2} e^{\c_2\sqrt{\q_2}\h_1^{(3)}+\hat{a}^2/2} \lp\erf \lp \frac{(\hat{a} + \hat{b})}{\sqrt{2}}\rp + 1\rp\rp \nonumber \\
 f_{(\c,u)}^{(2)} & = &    \frac{1}{2}e^{\c_2\sqrt{q}\h_1^{(3)}+\frac{\hat{a}^2}{2}}\lp \frac{2}{\sqrt{\pi}} e^{-\frac{(\hat{a} + \hat{b})^2}{2}}\lp\frac{d\hat{a}}{d\c_2}+\frac{d\hat{b}}{d\c_2}\rp\frac{1}{\sqrt{2}}\rp.
    \end{eqnarray}
One then analogously finds
 \begin{eqnarray}\label{eq:2levder14a1}
 \frac{df_{(zd)}^{(2,f)}}{d\q_2} = f_{(\c,d)}^{(1)}+f_{(\c,d)}^{(2)},
   \end{eqnarray}
   where
\begin{eqnarray}\label{eq:2levder14a2}
 f_{(\c,d)}^{(1)} & = &  \lp -\sqrt{\q_2}\h_1^{(3)}+2\nu+\hat{a} \frac{d\hat{a}}{d\c_2} \rp \lp \frac{1}{2} e^{-\c_2\sqrt{\q_2}\h_1^{(3)}+2\c_2\nu+\hat{a}^2/2} \lp\erf \lp \frac{(\hat{a} - \hat{b})}{\sqrt{2}}\rp + 1\rp\rp \nonumber \\
 f_{(\c,d)}^{(2)} & = &    \frac{1}{2}e^{-\c_2\sqrt{q}\h_1^{(3)}+2\c_2\nu+\frac{\hat{a}^2}{2}}\lp \frac{2}{\sqrt{\pi}} e^{-\frac{(\hat{a} - \hat{b})^2}{2}}\lp\frac{d\hat{a}}{d\c_2}-\frac{d\hat{b}}{d\c_2}\rp\frac{1}{\sqrt{2}}\rp.
    \end{eqnarray}
A combination of (\ref{eq:negprac24a2}), (\ref{eq:2levder12})-(\ref{eq:2levder14a2}) is then sufficient to determine $\c_2$--derivative.

\vspace{.1in}

\noindent \underline{\red{\textbf{(iv) $\gamma_{sq}$ -- derivative:}}} We easily find
\begin{eqnarray}\label{eq:2levder21a0}
   \frac{d\bar{\psi}_{rd}^{(2)}(\p,\q,\c,\gamma_{sq},\nu,\delta) }{d\gamma_{sq}}
  &  =  &   1 + \alpha \lp -\frac{1}{\c_2(2\gamma_{sq}-\c_2(1-\p_2))}+\frac{1}{2\c_2\gamma_{sq}}-\frac{\p_2}{(2\gamma_{sq}-\c_2(1-\p_2))^2}\rp \nonumber \\
   & = &    1 + \alpha \lp -\frac{1-\p_2}{2\gamma_{sq}(2\gamma_{sq}-\c_2(1-\p_2))}-\frac{\p_2}{(2\gamma_{sq}-\c_2(1-\p_2))^2}\rp. \nonumber \\
\end{eqnarray}

\vspace{.1in}

\noindent \underline{\red{\textbf{(v) $\nu$ -- derivative:}}} Following the path traced above, we start by writing
\begin{equation}\label{eq:2levder22}
   \frac{d\bar{\psi}_{rd}^{(2)}(\p,\q,\c,\gamma_{sq},\nu,\delta) }{d\nu}
=
 -2\delta
+ \frac{1}{\c_2}  \mE_{{\mathcal U}_3} \lp \frac{1}{f_{(zt)}^{(2,f)}} \frac{d\lp f_{(zt)}^{(2,f)}\rp}{d\nu} \rp.
 \end{equation}
Recalling once again on (\ref{eq:negprac24a2}), we have
\begin{eqnarray}\label{eq:2levder23}
 \frac{df_{(zt)}^{(2,f)}}{d\nu}& = &  \frac{df_{(zd)}^{(2,f)}}{d\nu}+ \frac{df_{(zu)}^{(2,f)}}{d\nu}.
   \end{eqnarray}
After finding
\begin{eqnarray}\label{eq:2levder24amin1}
\frac{d\hat{b}}{d\c_2} & = & -\frac{1}{\sqrt{1-\q_2}} \nonumber \\
\frac{d\hat{a}}{d\c_2} & = & 0,
\end{eqnarray}
we then also observe that
\begin{eqnarray}\label{eq:2levder24}
 \frac{df_{(zu)}^{(2,f)}}{d\nu} = f_{(\nu,u)}^{(1)}+f_{(\nu,u)}^{(2)},
   \end{eqnarray}
   where
\begin{eqnarray}\label{eq:2levder24a0}
 f_{(\nu,u)}^{(1)} & = &  \lp \hat{a} \frac{d\hat{a}}{d\c_2} \rp \lp \frac{1}{2} e^{\c_2\sqrt{\q_2}\h_1^{(3)}+\hat{a}^2/2} \lp\erf \lp \frac{(\hat{a} + \hat{b})}{\sqrt{2}}\rp + 1\rp\rp \nonumber \\
 f_{(\nu,u)}^{(2)} & = &    \frac{1}{2}e^{\c_2\sqrt{q}\h_1^{(3)}+\frac{\hat{a}^2}{2}}\lp \frac{2}{\sqrt{\pi}} e^{-\frac{(\hat{a} + \hat{b})^2}{2}}\lp\frac{d\hat{a}}{d\nu}+\frac{d\hat{b}}{d\nu}\rp\frac{1}{\sqrt{2}}\rp.
    \end{eqnarray}
In a similar manner, one then also finds
 \begin{eqnarray}\label{eq:2levder24a1}
 \frac{df_{(zd)}^{(2,f)}}{d\nu} = f_{(\nu,d)}^{(1)}+f_{(\nu,d)}^{(2)},
   \end{eqnarray}
   where
\begin{eqnarray}\label{eq:2levder24a2}
 f_{(\nu,d)}^{(1)} & = &  \lp 2\c_2+\hat{a} \frac{d\hat{a}}{d\c_2} \rp \lp \frac{1}{2} e^{-\c_2\sqrt{\q_2}\h_1^{(3)}+2\c_2\nu+\hat{a}^2/2} \lp\erf \lp \frac{(\hat{a} - \hat{b})}{\sqrt{2}}\rp + 1\rp\rp \nonumber \\
 f_{(\nu,d)}^{(2)} & = &    \frac{1}{2}e^{-\c_2\sqrt{q}\h_1^{(3)}+2\c_2\nu+\frac{\hat{a}^2}{2}}\lp \frac{2}{\sqrt{\pi}} e^{-\frac{(\hat{a} - \hat{b})^2}{2}}\lp\frac{d\hat{a}}{d\nu}-\frac{d\hat{b}}{d\nu}\rp\frac{1}{\sqrt{2}}\rp.
    \end{eqnarray}
  Combining (\ref{eq:negprac24a2}), (\ref{eq:2levder22})-(\ref{eq:2levder24a2}) is sufficient to determine $\nu$--derivative. After solving the following system
\begin{eqnarray}\label{eq:2levder32}
  \frac{d\bar{\psi}_{rd}^{(2)}(\p,\q,\c,\gamma_{sq},\nu,\delta) }{d\q_2}
 & = &  0\nonumber \\
 \frac{d\bar{\psi}_{rd}^{(2)}(\p,\q,\c,\gamma_{sq},\nu,\delta) }{d\p_2}
 & = &  0 \nonumber \\
 \frac{d\bar{\psi}_{rd}^{(2)}(\p,\q,\c,\gamma_{sq},\nu,\delta) }{d\c_2}
 & = &  0\nonumber \\
 \frac{d\bar{\psi}_{rd}^{(2)}(\p,\q,\c,\gamma_{sq},\nu,\delta) }{d\gamma_{sq}}
 & = &  0\nonumber \\
 \frac{d\bar{\psi}_{rd}^{(2)}(\p,\q,\c,\gamma_{sq},\nu,\delta) }{d\nu}
 & = &  0,
      \end{eqnarray}
and denoting by $\hat{\p}_2,\hat{\q}_2,\hat{\c}_2,\hat{\gamma}_{sq},\hat{\nu}$  the obtained solution, one obtains
$\bar{\psi}_{rd}^{(2)}(\hat{\p},\hat{\q},\hat{\c},\hat{\gamma}_{sq},\hat{\nu},\delta)$ and consequently $\xi(\delta)$ and $\xi_1(\delta)$. To determine the point of infliction one also needs $\delta$-derivative.

\vspace{.1in}

\noindent \underline{\red{\textbf{(vi) $\delta$ -- derivative:}}} We first quickly note
\begin{equation}\label{eq:2levder42}
   \frac{d\bar{\psi}_{rd}^{(2)}(\p,\q,\c,\gamma_{sq},\nu,\delta) }{d\delta} = -2\nu.
 \end{equation}
Moreover, recalling on (\ref{eq:ex0a12}) and (\ref{eq:negthmprac1eq2}), we also have
\begin{equation}\label{eq:2levder43}
\frac{d\xi_1(\delta)}{d\delta}=4(1-2\delta)-2\xi(\delta)\frac{d\xi(\delta)}{d\delta}
= 4(1-2\delta)-2\bar{\psi}_{rd}^{(2)}(\hat{\p},\hat{\q},\hat{\c},\hat{\gamma}_{sq},\hat{\nu},\delta)
\frac{d\bar{\psi}_{rd}^{(2)}(\hat{\p},\hat{\q},\hat{\c},\hat{\gamma}_{sq},\hat{\nu},\delta)}{d\delta}.
 \end{equation}
One also recognizes that
\begin{eqnarray}\label{eq:2levder44}
 \frac{d\bar{\psi}_{rd}^{(2)}(\hat{\p},\hat{\q},\hat{\c},\hat{\gamma}_{sq},\hat{\nu},\delta)}{d\delta}
 & = & \Bigg ( \Bigg.
  \frac{d\bar{\psi}_{rd}^{(2)}(\p,\q,\c,\gamma_{sq},\nu,\delta)}{d\p_2}\frac{d\p_2}{d\delta}
+  \frac{d\bar{\psi}_{rd}^{(2)}(\p,\q,\c,\gamma_{sq},\nu,\delta)}{d\q_2}\frac{d\q_2}{d\delta} \nonumber \\
& & +  \frac{d\bar{\psi}_{rd}^{(2)}(\p,\q,\c,\gamma_{sq},\nu,\delta)}{d\c_2}\frac{d\c_2}{d\delta}
+  \frac{d\bar{\psi}_{rd}^{(2)}(\p,\q,\c,\gamma_{sq},\nu,\delta)}{d\gamma_{sq}}\frac{d\gamma_{sq}}{d\delta} \nonumber \\
& &  +  \frac{d\bar{\psi}_{rd}^{(2)}(\p,\q,\c,\gamma_{sq},\nu,\delta)}{d\nu}\frac{d\nu}{d\delta}
+  \frac{d\bar{\psi}_{rd}^{(2)}(\p,\q,\c,\gamma_{sq},\nu,\delta)}{d\delta}
 \Bigg. \left.\Bigg )\right|_{(\p,\q,\c,\gamma_{sq},\nu)=(\hat{\p},\hat{\q},\hat{\c},\hat{\gamma}_{sq},\hat{\nu})} \nonumber \\
 & = & \left.\frac{d\bar{\psi}_{rd}^{(2)}(\p,\q,\c,\gamma_{sq},\nu,\delta)}{d\delta}
  \right|_{(\p,\q,\c,\gamma_{sq},\nu)=(\hat{\p},\hat{\q},\hat{\c},\hat{\gamma}_{sq},\hat{\nu})} \nonumber \\
 & = & -2\hat{\nu}.
 \end{eqnarray}
Recalling on (\ref{eq:ex0a16}) and (\ref{eq:altcap1}), one has that the associative memory capacity in the AGS basin sense is given as
 \begin{eqnarray}
 \alpha_c^{(AGS)} & \triangleq & \max \left \{\alpha |\hspace{.08in} \exists \delta\in \lp0,\frac{1}{2}\rp, \frac{d\xi_1(\delta)}{d\delta}=0\right \}.
  \label{eq:2levaltcap1}
\end{eqnarray}
Combining (\ref{eq:2levder43}) and (\ref{eq:2levaltcap1}), we then have
 \begin{eqnarray}
 \alpha_c^{(AGS,2)} & \triangleq & \max \left \{\alpha |\hspace{.08in} \exists \delta\in \lp0,\frac{1}{2}\rp,
 4(1-2\delta)+4\hat{\nu}\bar{\psi}_{rd}(\hat{\p},\hat{\q},\hat{\c},\hat{\gamma}_{sq},\hat{\nu},\delta)
 =0\right \}.
  \label{eq:2levaltcap2}
\end{eqnarray}
Taking the concrete numerical values gives
\vspace{.05in}
\begin{equation}\label{eq:2levder34a0}
\hspace{-2.5in}(\mbox{\bl{\textbf{second level:}}}) \qquad \qquad \qquad \qquad \qquad  a_c^{(AGS,2)}
\approx  \bl{\mathbf{0.138186}}.
  \end{equation}

\vspace{.1in}

\noindent \underline{\textbf{\emph{Closed form relations:}}} Following \cite{Stojnicnegsphflrdt23,Stojnictcmspnncapdinfdiffactrdt23}, we uncover the existence of closed form explicit parameters relations that turn out to be of crucial help for numerical handling of the above system. Namely, from (\ref{eq:2levder1}), we first find
\begin{eqnarray}\label{eq:helprel1}
 \q_2=\alpha\frac{\p_2}{(2\gamma_{sq}-\c_2(1-\p_2))^2}.
     \end{eqnarray}
From (\ref{eq:2levder21a0}), we further have
\begin{eqnarray}\label{eq:helprel2}
     1=\alpha\frac{1-\p_2}{2\gamma_{sq}(2\gamma_{sq}-\c_2(1-\p_2))}+\alpha\frac{\p_2}{(2\gamma_{sq}-\c_2(1-\p_2))^2}.
\end{eqnarray}
A combination of (\ref{eq:helprel1}) and (\ref{eq:helprel2}) gives
\begin{eqnarray}\label{eq:helprel3}
     1=\alpha\frac{1-\p_2}{2\gamma_{sq}(2\gamma_{sq}-\c_2(1-\p_2))}+\alpha\frac{\p_2}{(2\gamma_{sq}-\c_2(1-\p_2))^2}
     =\alpha\frac{1-\p_2}{2\gamma_{sq}(2\gamma_{sq}-\c_2(1-\p_2))}+\q_2.
\end{eqnarray}
Combining (\ref{eq:helprel1}) and (\ref{eq:helprel3}) further, we obtain
\begin{eqnarray}\label{eq:helprel4}
      \gamma_{sq} =\alpha\frac{1-\p_2}{2(1-\q_2)(2\gamma_{sq}-\c_2(1-\p_2))}
      =\frac{1}{2}\frac{1-\p_2}{1-\q_2}\sqrt{\frac{\q_2}{\p_2}\alpha}.
\end{eqnarray}
Also, from (\ref{eq:helprel1}) we can write
\begin{eqnarray}\label{eq:helprel5}
\c_2(1-\p_2) =2\gamma_{sq}-\sqrt{\frac{\p_2}{\q_2}\alpha}.
     \end{eqnarray}
A combination of (\ref{eq:helprel4}) and (\ref{eq:helprel5}) then gives
\begin{eqnarray}\label{eq:helprel6}
\c_2(1-\p_2) =2\gamma_{sq}-\sqrt{\frac{\p_2}{\q_2}\alpha}=\frac{1-\p_2}{1-\q_2}\sqrt{\frac{\q_2}{\p_2}\alpha}-\sqrt{\frac{\p_2}{\q_2}\alpha},
     \end{eqnarray}
and
\begin{eqnarray}\label{eq:helprel7}
\c_2 =\frac{1}{1-\q_2}\sqrt{\frac{\q_2}{\p_2}\alpha}-\frac{1}{1-\p_2}\sqrt{\frac{\p_2}{\q_2}\alpha}.
     \end{eqnarray}

\noindent \underline{\textbf{\emph{Concrete numerical values:}}}  In  Table \ref{tab:tab1}, $a_c^{(AGS,2)}$ is complemented with the concrete values of all the relevant quantities related to the second \emph{full} (2-sfl RDT) level of lifting. We also show in parallel the same quantities for the first \emph{full} (1-sfl RDT) which enables a systematic view of the lifting progress.
\begin{table}[h]
\caption{$r$-sfl RDT parameters; Hopfield associative memory capacity;  $\hat{\c}_1\rightarrow 1$; $n,\beta\rightarrow\infty$}\vspace{.1in}
\centering
\def\arraystretch{1.2}
\begin{tabular}{||l||c|c|c||c|c||c|c||c||c||}\hline\hline
 \hspace{-0in}$r$-sfl RDT                                             & $\hat{\gamma}_{sq}$    & $\hat{\nu}$   & $\hat{\delta}$ &  $\hat{\p}_2$ & $\hat{\p}_1`$     & $\hat{\q}_2$  & $\hat{\q}_1$ &  $\hat{\c}_2$    & $\alpha_c^{(AGS,r)}$  \\ \hline\hline
$1$-sfl RDT                                      & $0.1867$ & $-2.1372$ & $0.0163$ & $0$  & $\rightarrow 1$   & $0$ & $\rightarrow 1$
 &  $\rightarrow 0$  & \bl{$\mathbf{0.137906}$} \\ \hline\hline
   $2$-sfl RDT                                      & $0.2153$  & $-2.1252$ &  $0.0167$ & $0.99645$  & $\rightarrow 1$ &  $0.99694$ & $\rightarrow 1$
 &  $16.6192$   & \bl{$\mathbf{0.138186}$}  \\ \hline\hline
  \end{tabular}
\label{tab:tab1}
\end{table}

As was done on the first level, we visualize the above results by showing in Figure \ref{fig:fig1} $\xi_{tot}(\delta)$ which is defined as
\begin{equation}\label{eq:negprac21a0}
  \xi_{tot}(\delta)=\xi_1(\delta)-\xi_1(0)=-(1-2\delta)^2-\xi(\delta)^2+1+\alpha.
\end{equation}
We recall that $ \xi_{tot}(\delta)$ is the shifted version of $\xi_1(\delta)$ which accounts for the difference with respect to the energy of the aimed memorized pattern.
\begin{figure}[h]
\centering
\centerline{\includegraphics[width=1\linewidth]{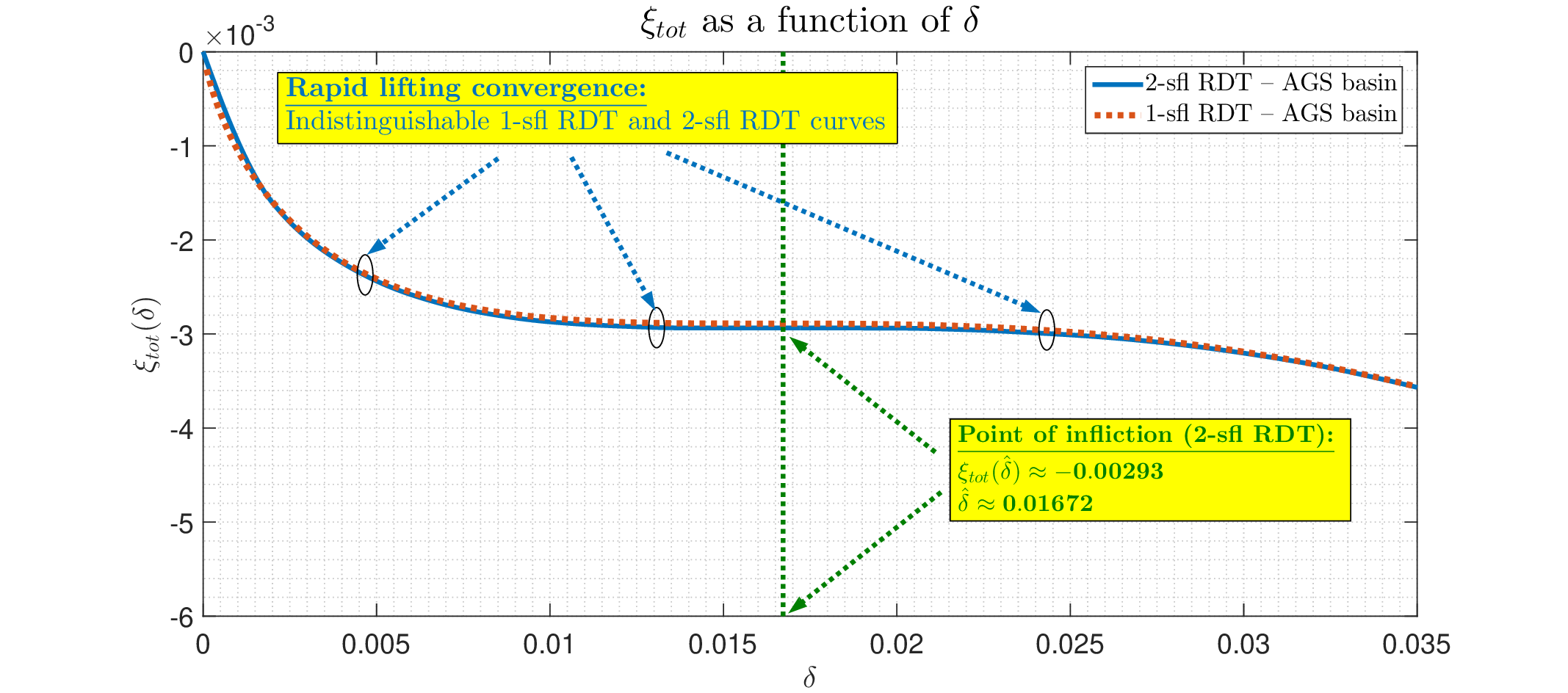}}
\caption{$\xi_{tot}$ as a function of $\delta$; $\alpha_c^{(AGS,2)}
\approx  \bl{\mathbf{0.138186}}$ -- maximum $\alpha$ such that the infliction point still exists on the second level of lifting}
\label{fig:fig2}
\end{figure}

As the figure shows, the convergence is remarkably fast. Even with a rather tiny scaling of $y$ axis (with the values $\sim 10^{-3}$), the $\xi_{tot}(\delta)\triangleq\xi_1(\delta)-\xi_1(0)$ curves for the first and second level are almost indistinguishable. Consequently, we have that already on the second level of lifting the associative capacity relative improvement is $\sim 0.1\%$, which renders evaluations for higher levels practically irrelevant. We also observe that the obtained results exactly match the ones obtained utilizing the statistical physics methods based on replica symmetry \cite{AmiGutSom85} and replica symmetry breaking \cite{SteKuh94}.

\subsubsection{Modulo-$\m$ sfl RDT}
\label{sec:posmodm}

Everything that we presented above can be repeated if one relies on the so-called modulo-$m$ sfl RDT frame of \cite{Stojnicsflgscompyx23}. Instead of Theorem \ref{thme:negthmprac1}, one then has the following theorem.
 \begin{theorem}
  \label{thme:negthmprac2}
  Assume the setup of Theorem \ref{thme:negthmprac1} and instead of the complete, assume the modulo-$\m$ sfl RDT setup of \cite{Stojnicsflgscompyx23}.
   Let the ``fixed'' parts of $\hat{\p}$, $\hat{\q}$, and $\hat{\c}$ satisfy $\hat{\p}_1\rightarrow 1$, $\hat{\q}_1\rightarrow 1$, $\hat{\c}_1\rightarrow 1$, $\hat{\p}_{r+1}=\hat{\q}_{r+1}=\hat{\c}_{r+1}=0$, and let the ``non-fixed'' parts of $\hat{\p}_k$, and $\hat{\q}_k$ ($k\in\{2,3,\dots,r\}$), $\hat{\gamma}_{sq}$, and $\hat{\nu}$ be the solutions of the following system of equations
  \begin{eqnarray}\label{eq:negthmprac2eq1}
   \frac{d \bar{\psi}_{rd}(\p,\q,\c,\gamma_{sq},\nu,\delta)}{d\p} =  0 \nonumber \\
   \frac{d \bar{\psi}_{rd}(\p,\q,\c,\gamma_{sq},\nu,\delta)}{d\q} =  0 \nonumber \\
    \frac{d \bar{\psi}_{rd}(\p,\q,\c,\gamma_{sq},\nu,\delta)}{d\gamma_{sq}} =  0\nonumber \\
    \frac{d \bar{\psi}_{rd}(\p,\q,\c,\gamma_{sq},\nu,\delta)}{d\nu} =  0.
 \end{eqnarray}
 Consequently, let
\begin{eqnarray}\label{eq:negthmprac2eq2}
c_k(\hat{\p},\hat{\q})  & = & \sqrt{\hat{\q}_{k-1}-\hat{\q}_k} \nonumber \\
b_k(\hat{\p},\hat{\q})  & = & \sqrt{\hat{\p}_{k-1}-\hat{\p}_k}.
 \end{eqnarray}
 Then
 \begin{eqnarray}
\xi(\delta) \geq   \min_{\c} \bar{\psi}_{rd}(\hat{\p},\hat{\q},\c,\hat{\gamma}_{sq},\hat{\nu},\delta).
  \label{eq:negthmprac2eq3}
\end{eqnarray}
\end{theorem}
\begin{proof}
Follows from the above discussions, Theorems \ref{thm:thmsflrdt1} and \ref{thme:negthmprac1}, Corollary \ref{cor:cor1}, and the sfl RDT machinery presented in \cite{Stojnicnflgscompyx23,Stojnicsflgscompyx23,Stojnicflrdt23}.
\end{proof}
We conducted the numerical evaluations using the modulo-$\m$ results of the above theorem and found no scenario where the inequality in (\ref{eq:negthmprac2eq3}) is not tight, which indicates that the \emph{stationarity} over $\c$ is of the \emph{minimization} type. Moreover, we repeated the same type of evaluations for the NLT considerations discussed below and observed the very same behavior.

\subsection{Numerical evaluations -- NLT basin}
\label{sec:nuemricalagsNLT}

Below we show that the capacity results analogous to the AGS basin ones presented above, can be obtained for the NLT basin as well. Moreover, many of the calculations remains the same. What changes is the way how one uses them. As was the case above when analyzing the AGS basin, we again start the evaluations with $r=1$ and proceed by incrementally increasing $r$.

\subsubsection{$r=1$ -- first level of lifting}
\label{sec:firstlevNLT}

Everything between (\ref{eq:negprac19}) and (\ref{eq:negprac19a9}) remains in place again, which means that we have $\hat{\nu}=\sqrt{2} \erfinv \lp 1-2\delta\rp$,
\begin{eqnarray}
\xi_1(\delta)
& = & -\lp 1-2\delta\rp^2 -
 \lp  \frac{2}{\sqrt{2\pi}} e^{-\lp \erfinv\lp 1-2\delta \rp\rp^2} +\sqrt{\alpha} \rp^2,
 \label{eq:negprac19a6nl}
\end{eqnarray}
 and
\begin{eqnarray}
\frac{d\xi_1(\delta)}{d\delta}
 & = & 4\lp 1-2\delta\rp - 4 \sqrt{2} \erfinv \lp 1-2\delta\rp \lp  \frac{2}{\sqrt{2\pi}} e^{-\lp \erfinv\lp 1-2\delta \rp\rp^2} +\sqrt{\alpha} \rp.
 \label{eq:negprac19a9nl}
\end{eqnarray}
Recalling on (\ref{eq:ex0a16}), we observe that the associative memory capacity in the NLT basin sense, can alternatively be obtained as
 \begin{eqnarray}
 \alpha_c^{(NLT)} & \triangleq & \max \left \{\alpha |\hspace{.08in} \exists \hat{\delta}\in \lp0,\frac{1}{2}\rp, \left.\frac{d\xi_1(\delta)}{d\delta}\right|_{\delta=\hat{\delta}}=0\quad \mbox{and}\quad \xi_1(\hat{\delta})=\xi_1(0)\right \}.
  \label{eq:altcap1nl}
\end{eqnarray}
Utilizing (\ref{eq:negprac19a9nl}) and  $\frac{d\xi_1(\delta)}{d\delta}=0$, implies the following choice
\begin{eqnarray}
 \alpha & = & \lp  \frac{1-2\delta}{\sqrt{2} \erfinv \lp 1-2\delta\rp} -   \frac{2}{\sqrt{2\pi}} e^{-\lp \erfinv\lp 1-2\delta \rp\rp^2}\rp^2.
 \label{eq:altcap2nl}
\end{eqnarray}
Recalling on (\ref{eq:negprac19a6nl}) and $\xi(0)=-1-\alpha$, we then have that $\xi_1(\hat{\delta})=\xi_1(0)$ implies
\begin{eqnarray}
 -\lp 1-2\delta\rp^2 -
 \lp  \frac{2}{\sqrt{2\pi}} e^{-\lp \erfinv\lp 1-2\delta \rp\rp^2} +\sqrt{\alpha} \rp^2
 = \xi_1(\delta)
 =\xi_1(0) =-1-\alpha.
 \label{eq:altcap3nl}
\end{eqnarray}
After plugging $\alpha$ from (\ref{eq:altcap2nl}) into (\ref{eq:altcap3nl}), we obtain
\begin{eqnarray}
 -\lp 1-2\delta\rp^2 -
 \lp \frac{1-2\delta}{\sqrt{2} \erfinv \lp 1-2\delta\rp}\rp^2
 =-1-\lp  \frac{1-2\delta}{\sqrt{2} \erfinv \lp 1-2\delta\rp} -   \frac{2}{\sqrt{2\pi}} e^{-\lp \erfinv\lp 1-2\delta \rp\rp^2}\rp^2.
 \label{eq:altcap4nl}
\end{eqnarray}
A bit of additional algebraic transformations gives
\begin{eqnarray}
2\delta(1-\delta)
 - \frac{1-2\delta}{\sqrt{\pi} \erfinv \lp 1-2\delta\rp}  e^{-\lp \erfinv\lp 1-2\delta \rp\rp^2}
 +\frac{1}{\pi} e^{-2\lp \erfinv\lp 1-2\delta \rp\rp^2}=0.
 \label{eq:altcap5nl}
\end{eqnarray}
After denoting by $\hat{\delta}$ the solution of (\ref{eq:altcap5nl}), we then from (\ref{eq:altcap2nl}) finally have
 \begin{equation}\label{eq:negprac21nl}
\hspace{-0.2in}(\mbox{\bl{\textbf{first level:}}}) \qquad \qquad   \alpha_c^{(NLT,1)}
=
\lp  \frac{1-2\hat{\delta}}{\sqrt{2} \erfinv \lp 1-2\hat{\delta}\rp} -   \frac{2}{\sqrt{2\pi}} e^{-\lp \erfinv\lp 1-2\hat{\delta} \rp\rp^2}\rp^2
\approx  \bl{\mathbf{0.1294899}}.
  \end{equation}
Similarly to what we sowed in Figure \ref{fig:fig1}, in Figure \ref{fig:fig4}, we, for $\alpha=\alpha_c^{(NLT,1)}$, show $\xi_{tot}(\delta)$ -- a shifted version of $\xi_1(\delta)$ defined as
\begin{equation}\label{eq:negprac21a0nl}
  \xi_{tot}(\delta)=\xi_1(\delta)-\xi_1(0)=-(1-2\delta)^2-\xi(\delta)^2+1+\alpha.
\end{equation}
As figure indicates, in addition to $\xi_{tot}(0)=0$, one also has $\xi_{tot}(\delta)=0$ for $\delta=\hat{\delta}\approx 0.033935$.

 \begin{figure}[h]
\centering
\centerline{\includegraphics[width=1\linewidth]{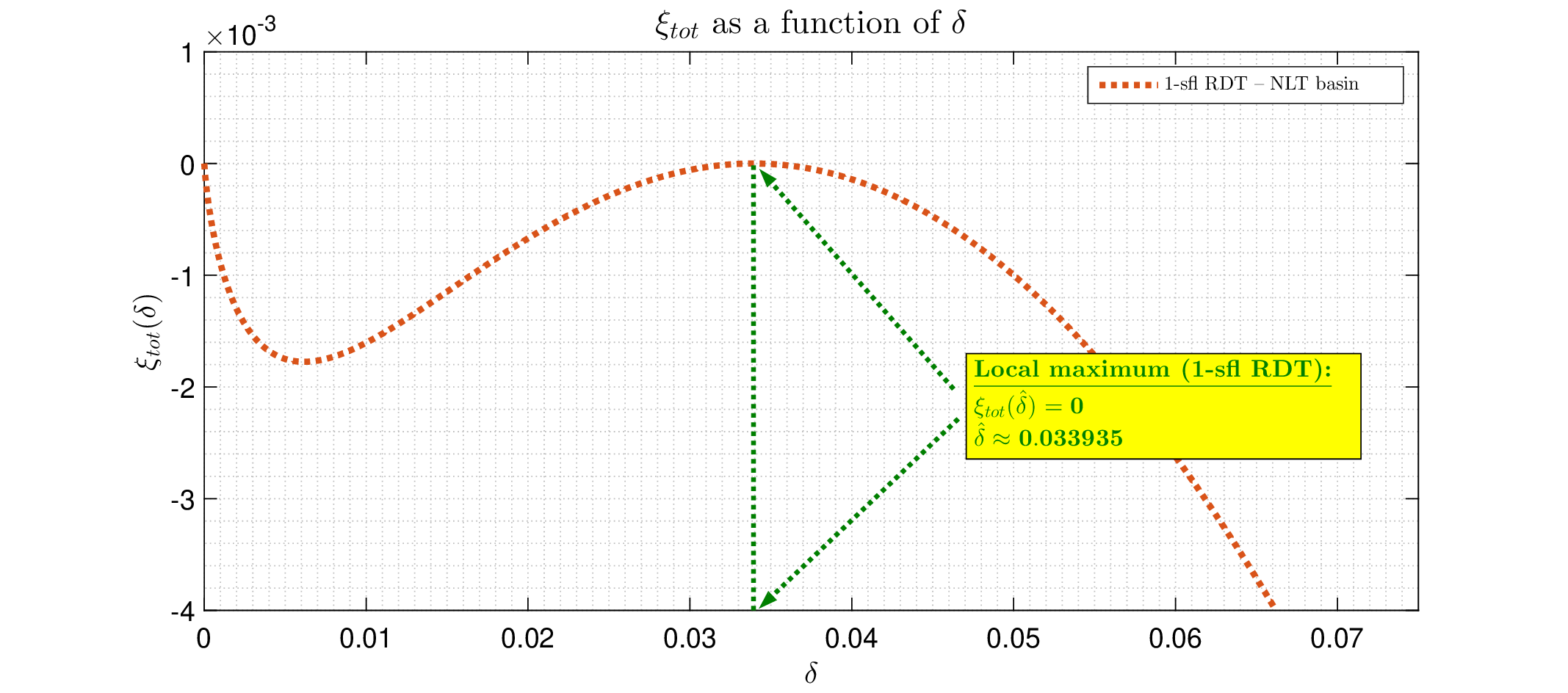}}
\caption{$\xi_{tot}$ as a function of $\delta$; $\alpha_c^{(NLT,1)}
\approx  \bl{\mathbf{0.1294899}}$ -- maximum $\alpha$ such that $\exists\delta\in \lp 0,\frac{1}{2}\rp$ for which  $\xi_{tot}(\delta)=\xi_{tot}(0)$ on the first level of lifting}
\label{fig:fig4}
\end{figure}

For the completeness, we also mention that with a little bit of additional algebraic transformations the above capacity characterization can be rewritten in the following, possibly more compact/elegant, way:
\begin{eqnarray}
   \alpha_c^{(NLT,1)} &  =  & \frac{\mbox{erf}(x)^2}{2x^2}-1+\mbox{erf}(x)^2 \approx \bl{\mathbf{0.1294899}}, \quad
1-\mbox{erf}(x)^2- \frac{2\mbox{erf}(x)e^{-x^2}}{\sqrt{\pi}x}+\frac{2e^{-2x^2}}{\pi}=0.
 \label{eq:repabs2}
\end{eqnarray}
It is not that difficult to see the correspondence $x=\erfinv\lp 1-2\hat{\delta}\rp$.

\subsubsection{$r=2$ -- second level of lifting}
\label{sec:secondlevnl}

One again can utilize the results obtained for the AGS basin. In particular, everything between (\ref{eq:negprac24}) and (\ref{eq:2levder44})
remain in place. Instead of utilizing (\ref{eq:2levaltcap1}), we now recall on (\ref{eq:ex0a16}) and utilize (\ref{eq:altcap1nl}). As the difference between (\ref{eq:2levaltcap1}) and (\ref{eq:altcap1nl}) is in an additionally imposed constraint $\xi_1(\hat{\delta})=\xi_1(0)$, one can still utilize  (\ref{eq:2levaltcap2}) but with a slight modification to account for such a constraint. This basically means that instead of (\ref{eq:2levaltcap1}), we now have
 \begin{eqnarray}
 \alpha_c^{(NLT)} & \triangleq & \max \left \{\alpha |\hspace{.08in} \exists \hat{\delta}\in \lp0,\frac{1}{2}\rp, \left.\frac{d\xi_1(\delta)}{d\delta}\right|_{\delta=\hat{\delta}}=0\quad \mbox{and}\quad \xi_1(\hat{\delta})=\xi_1(0)=-1-\alpha\right \},
  \label{eq:altcap1nl}
\end{eqnarray}
and instead of (\ref{eq:2levaltcap2})
  \begin{equation}
 \alpha_c^{(NLT,2)} =  \max \left \{\alpha |\hspace{.04in} \exists \delta\in \lp0,\frac{1}{2}\rp,
 4(1-2\delta)+4\hat{\nu} \bar{\psi}_{rd}^{(2)}(\hat{\p},\hat{\q},\hat{\c},\hat{\gamma}_{sq},\hat{\nu},\delta)
 =0\quad \mbox{and} \quad \xi_1(\hat{\delta})=\xi_1(0)=-1-\alpha \right \} ,
  \label{eq:2levaltcap2nl}
\end{equation}
where we also recall from (\ref{eq:ex0a12}) and (\ref{eq:negthmprac1eq2})  that on the second level $\xi_1(\delta)$ is given as
\begin{eqnarray}
\xi_1(\delta)
 & = & -\lp 1-2\delta\rp^2 -\xi(\delta)^2 = -\lp 1-2\delta\rp^2 -\lp \bar{\psi}_{rd}^{(2)}(\hat{\p},\hat{\q},\hat{\c},\hat{\gamma}_{sq},\hat{\nu},\delta) \rp^2, \label{eq:ex0a12nl}
\end{eqnarray}
Taking the concrete numerical values gives
\vspace{.05in}
\begin{equation}\label{eq:2levder34a0}
\hspace{-2.5in}(\mbox{\bl{\textbf{second level:}}}) \qquad \qquad \qquad \qquad \qquad  a_c^{(NLT,2)}
\approx  \bl{\mathbf{0.12979}}.
  \end{equation}

\noindent \underline{\textbf{\emph{Concrete numerical values:}}}  Analogously to  Table \ref{tab:tab1}, we in Table \ref{tab:tab2} complement $a_c^{(NLT,2)}$ with the concrete values of all the relevant quantities related to the second \emph{full} (2-sfl RDT) level of lifting. As was the case for the AGS basin in Table \ref{tab:tab1}, we here also show both first and second level parameters' values in parallel ensuring that a systematic view of the lifting progressing mechanism is enabled.
\begin{table}[h]
\caption{$r$-sfl RDT parameters; Hopfield associative memory capacity -- NLT basin;  $\hat{\c}_1\rightarrow 1$; $n,\beta\rightarrow\infty$}\vspace{.1in}
\centering
\def\arraystretch{1.2}
\begin{tabular}{||l||c|c|c||c|c||c|c||c||c||}\hline\hline
 \hspace{-0in}$r$-sfl RDT                                             & $\hat{\gamma}_{sq}$    & $\hat{\nu}$   & $\hat{\delta}$ &  $\hat{\p}_2$ & $\hat{\p}_1`$     & $\hat{\q}_2$  & $\hat{\q}_1$ &  $\hat{\c}_2$    & $\alpha_c^{(NLT,r)}$  \\ \hline\hline
$1$-sfl RDT                                      & $0.1799$ & $-1.8259$ & $0.0339$ & $0$  & $\rightarrow 1$   & $0$ & $\rightarrow 1$
 &  $\rightarrow 0$  & \bl{$\mathbf{0.12949}$} \\ \hline\hline
   $2$-sfl RDT                                      & $0.2309$  & $-1.8111$ &  $0.0347$ & $0.98806$  & $\rightarrow 1$ &  $0.99067$ & $\rightarrow 1$
 &  $8.54157$   & \bl{$\mathbf{0.12979}$}  \\ \hline\hline
  \end{tabular}
\label{tab:tab2}
\end{table}

As earlier, we visualize the above results in Figure \ref{fig:fig5}  by showing $\xi_{tot}(\delta)$.
\begin{figure}[h]
\centering
\centerline{\includegraphics[width=1\linewidth]{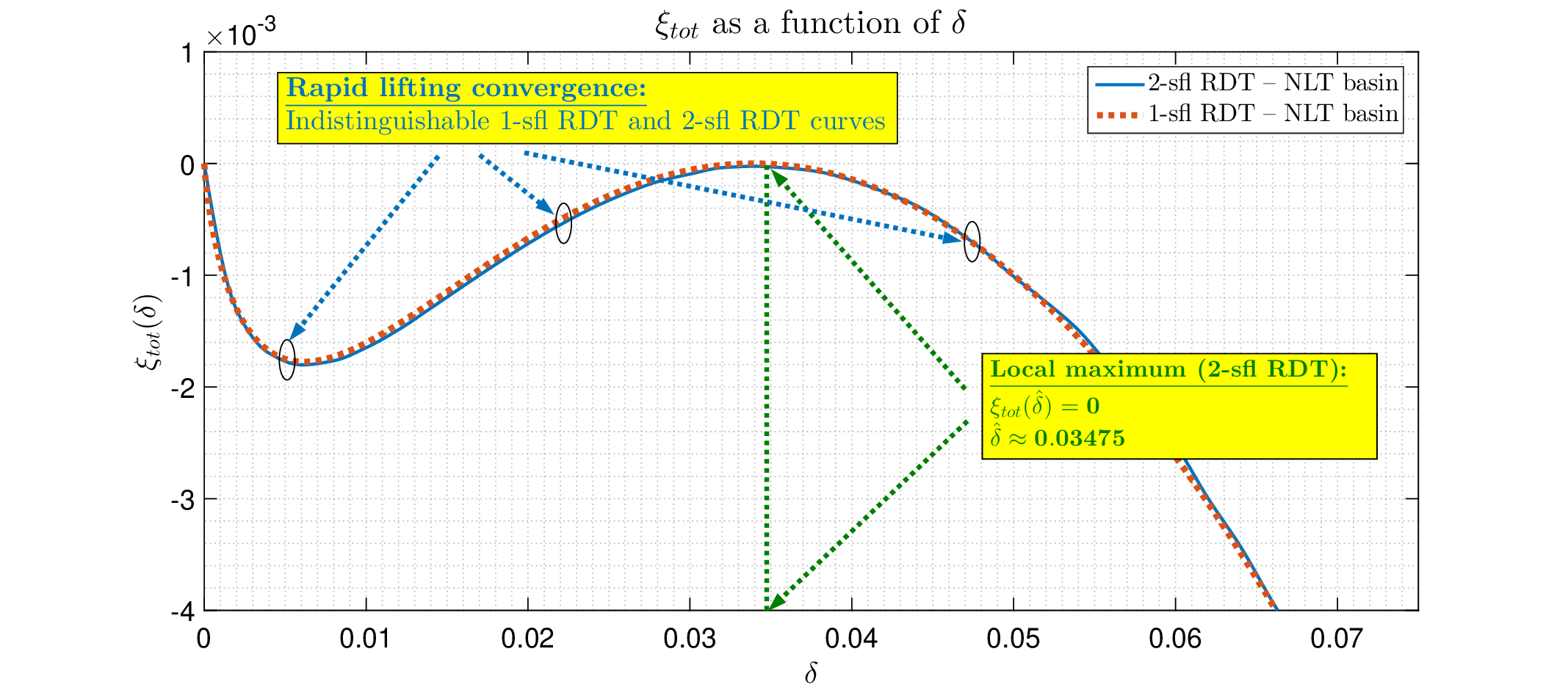}}
\caption{$\xi_{tot}$ as a function of $\delta$; $\alpha_c^{(NLT,2)}
\approx  \bl{\mathbf{0.12979}}$ -- maximum $\alpha$ such that $\exists\delta\in \lp 0,\frac{1}{2}\rp$ for which  $\xi_{tot}(\delta)=\xi_{tot}(0)$ on the second level of lifting}
\label{fig:fig5}
\end{figure}
As is rather clear from the figure, the convergence is remarkably fast. Similarly to what we had for the AGS basin, the $\xi_{tot}(\delta)\triangleq\xi_1(\delta)-\xi_1(0)$ curves for the first and second level are almost indistinguishable and already on the second level of lifting the associative capacity relative improvement is $\sim 0.1\%$. This again, for all practical purposes, renders evaluations for higher levels as basically of not much significance.

\subsection{AGS vs NLT basins}
\label{sec:agsvsnl}

To illustrate the key conceptual difference between AGS and NLT basins, we show their $\xi_{tot}(\delta)$ in parallel in Figure \ref{fig:fig6}.
\begin{figure}[h]
\centering
\centerline{\includegraphics[width=1\linewidth]{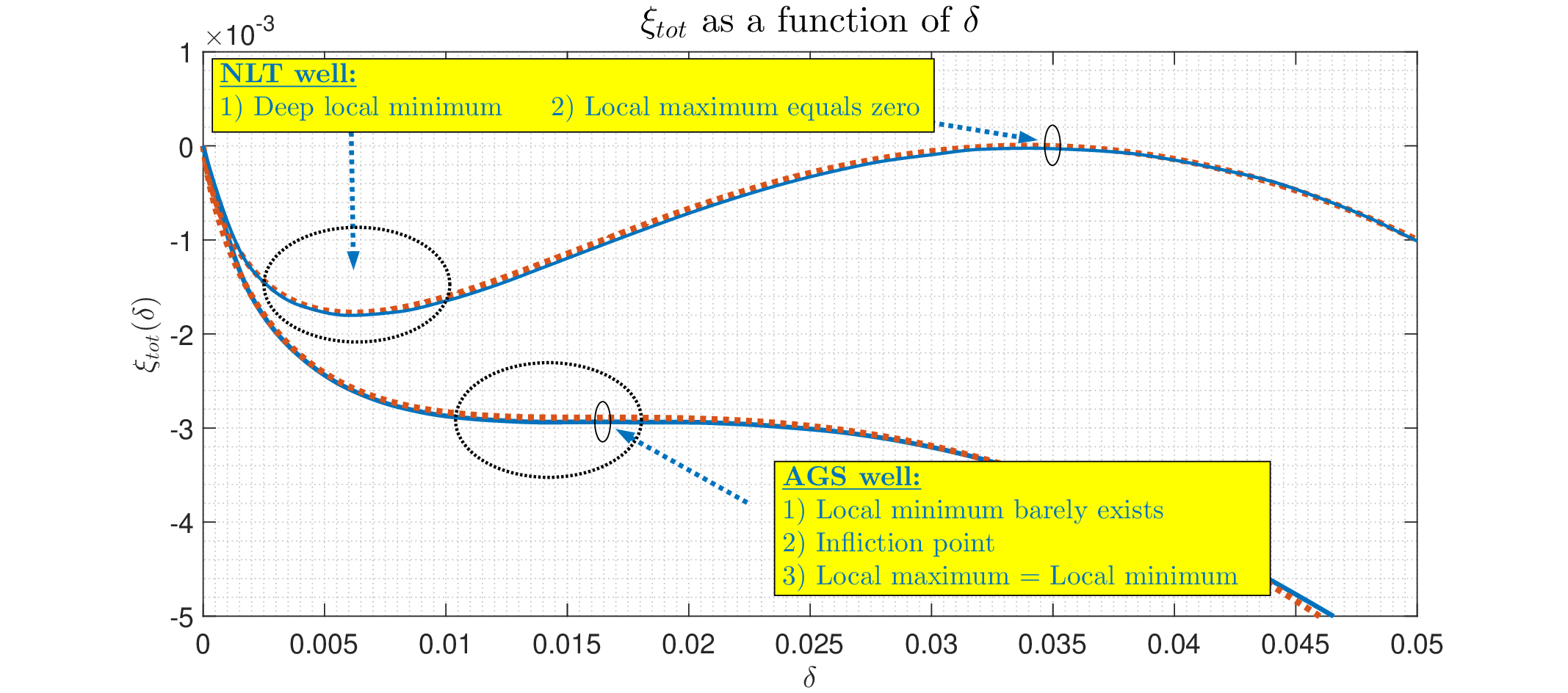}}
\caption{$\xi_{tot}$ as a function of $\delta$; AGS vs NLT basin}
\label{fig:fig6}
\end{figure}
As can be seen from the figure, the lifting convergence is remarkably fast in both cases. Also, the curves clearly display the key point: the AGS basin revolves around the concept of having any form of a well (no matter how shallow it is) whereas the NLT insists on having a more robust rather deep well with the local maximum matching (borderline exceeding) the zero value that corresponds to the value of the aimed stored pattern.

\subsection{Global local minima}
\label{sec:globloc}

Both AGS and NLT ultimately rely on having basins of attraction around the aimed pattern which are characterized by the existence of a local minimum. As $\alpha$ decreases, these local minima become global. The above machinery can be reutilized to determine critical $\alpha=\alpha_c^{(GLM)}$ where this happens. Namely, this will happen where no vector $\x$ which is not in the vicinity of one of the patterns (i.e., which has zero overlap with them) can achieve a minimum lower than the local one from the AGS well. This basically means that one can use all of the above AGS results with the exception that now $\delta=0.5$.

In particular, we have for the first level of lifting
\begin{equation}
 \alpha_c^{(GLM,1)} =  \left \{ \alpha |  \min_{\delta\in \lp 0,\frac{1}{2}\rp }\xi_1(\delta)-\xi_1(0.5) =0 \right \},
 \label{eq:altcap2glm}
\end{equation}
where $\xi_1(\delta)$ is as in (\ref{eq:negprac19a6}), i.e.,
\begin{eqnarray}
\xi_1(\delta)
 & = & -\lp 1-2\delta\rp^2 -
 \lp  \frac{2}{\sqrt{2\pi}} e^{-\lp \erfinv\lp 1-2\delta \rp\rp^2} +\sqrt{\alpha} \rp^2.
 \label{eq:altcap3glm}
\end{eqnarray}
We first note that
\begin{eqnarray}
\xi_1(0.5)
 & = &   -
 \lp  \frac{2}{\sqrt{2\pi}}   +\sqrt{\alpha} \rp^2,
 \label{eq:altcap3a0glm}
\end{eqnarray}
and then, recalling  on (\ref{eq:negprac19a9}), also
\begin{eqnarray}
\frac{d\xi_1(\delta)}{d\delta}
 & = & 4\lp 1-2\delta\rp - 4 \sqrt{2} \erfinv \lp 1-2\delta\rp \lp  \frac{2}{\sqrt{2\pi}} e^{-\lp \erfinv\lp 1-2\delta \rp\rp^2} +\sqrt{\alpha} \rp.
\label{eq:altcap3a1glm}
\end{eqnarray}
Denoting by $\hat{\delta}$ the solution of the optimization in (\ref{eq:altcap2glm}), one observes, after combining (\ref{eq:altcap2glm})-(\ref{eq:altcap3a1glm}),  that pair $\lp\alpha_c^{(GLM,1)},\hat{\delta}\rp$ satisfies the following system
\begin{eqnarray}
& &  -\lp 1-2\delta\rp^2 -
 \lp  \frac{2}{\sqrt{2\pi}} e^{-\lp \erfinv\lp 1-2\delta \rp\rp^2} +\sqrt{\alpha} \rp^2
 =\xi_1(\delta)
=
\xi_1(0.5)
  =    -
 \lp  \frac{2}{\sqrt{2\pi}}   +\sqrt{\alpha} \rp^2 \nonumber \\
& &  \frac{d\xi_1(\delta)}{d\delta}
  =  4\lp 1-2\delta\rp - 4 \sqrt{2} \erfinv \lp 1-2\delta\rp \lp  \frac{2}{\sqrt{2\pi}} e^{-\lp \erfinv\lp 1-2\delta \rp\rp^2} +\sqrt{\alpha} \rp = 0.
 \label{eq:altcap3a2glm}
\end{eqnarray}
From (\ref{eq:altcap3a2glm}), we then further find the following system
\begin{eqnarray}
\sqrt{\alpha}
& = &  (1-2\delta)\sqrt{1+\frac{1}{2(\erfinv(1-2\delta))^2}}-\frac{2}{\sqrt{2\pi}}
 \nonumber \\
  \sqrt{\alpha}
&  =  & \frac{\lp 1-2\delta\rp }{\sqrt{2} \erfinv \lp 1-2\delta\rp } -  \frac{2}{\sqrt{2\pi}} e^{-\lp \erfinv\lp 1-2\delta \rp\rp^2}.
\label{eq:altcap3a3glm}
\end{eqnarray}
From (\ref{eq:altcap3a3glm}) one then establishes that
\begin{eqnarray}
\hat{\delta}
 =  \left \{\delta| (1-2\delta)\lp \sqrt{1+\frac{1}{2(\erfinv(1-2\delta))^2}} - \frac{1 }{\sqrt{2} \erfinv \lp 1-2\delta\rp } \rp
 =\frac{2}{\sqrt{2\pi}} \lp 1- e^{-\lp \erfinv\lp 1-2\delta \rp\rp^2}\rp
  \right \}.
\label{eq:altcap3a4glm}
\end{eqnarray}
Utilizing (\ref{eq:altcap3a4glm})  we find $\hat{\delta}\approx 5.6574\times 10^{-6}$ and
\vspace{.0in}
\begin{equation}
\hspace{-.2in} \mbox{(\bl{\textbf{first level:}})} \qquad\qquad\qquad \alpha_c^{(GLM,1)}  =   \frac{\lp 1-2\hat{\delta}\rp }{\sqrt{2} \erfinv \lp 1-2\hat{\delta}\rp } -  \frac{2}{\sqrt{2\pi}} e^{-\lp \erfinv\lp 1-2\hat{\delta} \rp\rp^2} \approx \bl{\mathbf{0.05185}},
 \label{eq:altcap4glm}
\end{equation}
Moreover, using the above $\hat{\delta}$ and $\alpha_c^{(GLM,1)}$, one also finds $\xi_1(\hat{\delta})-\xi_1(0)=\xi_1(\hat{\delta})+1+\alpha_c^{(GLM,1)}=\xi_1(0.5)-\xi_1(0)=\xi_1(0.5)+1+\alpha_c^{(GLM,1)}\approx -1.072\times 10^{-6}$.

After conducting the corresponding numerical work for the second level of lifting, one obtains the results shown in Tables \ref{tab:tab3} and \ref{tab:tab4}. Table \ref{tab:tab3} relates to $\min_{\delta\in \lp 0,\frac{1}{2}\rp }\xi_1(\delta)$ and Table \ref{tab:tab4} relates to $\xi_1(0.5)$. We also observe that $\alpha_c^{(GLM,1)}$ exactly matches the value obtained through the replica symmetry methods in \cite{AmiGutSom85}.
\begin{table}[h]
\caption{$r$-sfl RDT parameters; Hopfield associative memory capacity -- Global local minima; minimum achieving $\delta\in\lp0,\frac{1}{2}\rp$;  $\hat{\c}_1\rightarrow 1$; $n,\beta\rightarrow\infty$}\vspace{.1in}
\centering
\def\arraystretch{1.2}
\begin{tabular}{||l||c|c|c||c|c||c|c||c||c||}\hline\hline
 \hspace{-0in}$r$-sfl RDT                                             & $\hat{\gamma}_{sq}$    & $\hat{\nu}$   & $\hat{\delta}$ &  $\hat{\p}_2$ & $\hat{\p}_1`$     & $\hat{\q}_2$  & $\hat{\q}_1$ &  $\hat{\c}_2$    & $\alpha_c^{(GLM,r)}$  \\ \hline\hline
$1$-sfl RDT                                      & $0.1139$ & $-4.3904$ & $5.6574e-06$ & $0$  & $\rightarrow 1$   & $0$ & $\rightarrow 1$
 &  $\rightarrow 0$  & \bl{$\mathbf{0.051854}$} \\ \hline\hline
   $2$-sfl RDT                                      & $0.1185$  & $-4.2184$ &  $1.23e-5$ &  $\rightarrow 1$  & $\rightarrow 1$ &  $\rightarrow 1$ & $\rightarrow 1$
 &  $\rightarrow 0$   & \bl{$\mathbf{0.056141}$}  \\ \hline\hline
  \end{tabular}
\label{tab:tab3}
\end{table}
\begin{table}[h]
\caption{$r$-sfl RDT parameters; Hopfield associative memory capacity -- Global local minima; $\delta=0.5$;  $\hat{\c}_1\rightarrow 1$; $n,\beta\rightarrow\infty$}\vspace{.1in}
\centering
\def\arraystretch{1.2}
\begin{tabular}{||l||c|c|c||c|c||c|c||c||c||}\hline\hline
 \hspace{-0in}$r$-sfl RDT                                             & $\hat{\gamma}_{sq}$    & $\hat{\nu}$   & $\hat{\delta}$ &  $\hat{\p}_2$ & $\hat{\p}_1`$     & $\hat{\q}_2$  & $\hat{\q}_1$ &  $\hat{\c}_2$    & $\alpha_c^{(GLM,r)}$  \\ \hline\hline
$1$-sfl RDT                                      & $0.1139$ & $0$ & $0.5$ & $0$  & $\rightarrow 1$   & $0$ & $\rightarrow 1$
 &  $\rightarrow 0$  & \bl{$\mathbf{0.051854}$} \\ \hline\hline
   $2$-sfl RDT                                      & $0.2200$  & $0$ &  $0.5$ &  $0.5625$  & $\rightarrow 1$ & $0.7314$ & $\rightarrow 1$
 &  $0.5308$   & \bl{$\mathbf{0.056141}$}  \\ \hline\hline
  \end{tabular}
\label{tab:tab4}
\end{table}

\section{Conclusion}
\label{sec:conc}

We studied the classical \emph{Hopfield} neural network model with the \emph{Hebbian} learning rule. The primary focus was on network's associative memory. We considered the Hopfield's original scenario where the retrieval dynamics is allowed to make a small fraction of errors in recovering the aimed patterns. Studying binary random patterns,  Hopfield suggested in his introductory paper, that the associative capacity (maximal number of reliably stored patterns, $m$) of fraction of errors allowing dynamics should grow linearly with the pattern's size, $n$. Moreover, utilizing a random noise argument and relying on numerical algorithmic evidence, he also famously predicted that $\alpha =\lim_{n\rightarrow\infty}\frac{m}{n}\approx 0.14$. We here consider several variants of this Hopfield setup and obtain results that are in an excellent agreement with his predictions

After first making the connection between studying Hopfield's nets capacities and  bilinearly indexed (bli) random processes, we utilized recent progress in bli's from \cite{Stojnicsflgscompyx23,Stojnicnflgscompyx23} and the \emph{fully lifted} random duality theory (fl RDT) from  \cite{Stojnicflrdt23} to create a powerful generic framework for the analysis of the networks dynamics. The concept of the associative capacity critically depends on the existence of a basin of attraction around each of the patterns to be memorized. For a proper assessment of the capacity an accurate description of the basin is unavoidably needed. We considered two notions of the basin most prevalently used in the literature: \textbf{\emph{(i)}} The AGS one from  \cite{AmiGutSom85} which relies on the existence of a local energy minimum around aimed pattern; and \textbf{\emph{(ii)}} The NLT one from \cite{Newman88,Louk94,Louk94a,Louk97,Tal98} which relies on the existence of a firm energy barrier around patterns. A tradeoff between the two exists as well. Namely, the NLT one is more strict and robust and implies the first one. On the other hand the capacities that it allows for are smaller than the ones allowed by the AGS.

To be able to successively utilize the fl RDT and obtain concrete capacity characterizations for both AGS and NLT basins, one needs to perform a substantial amount of numerical work. We performed all the needed evaluations, uncovered remarkable closed form explicit analytical relations among key lifting parameters, and ultimately obtained concrete numerical values for the capacities of both basin. Moreover, on the first level of lifting, we obtained explicit closed form capacity characterizations. After completing the second level numerical evaluations, we uncovered a surprising convergence property. Namely, while the convergence rate of fl RDT is generically excellent, we here discovered that it is remarkably rapid even for the typical fl RDT standards. In particular, we found for both AGS and NLT that the relative improvements no better than $\sim 0.1\%$ are happening already on the \textbf{\emph{second}} (first nontrivial) level of the full lifting. We also found that the obtained AGS based capacity characterizations exactly match those obtained through the statistical physics replica symmetry methods of \cite{AmiGutSom85} and the corresponding symmetry breaking ones of \cite{SteKuh94}. On the other hand the NLT ones are substantially higher than the previously best known ones of \cite{Newman88,Louk94,Louk94a,Louk97,Tal98}.

The developed methodology is very generic and  further extensions and generalizations are possible. They include studying various other Network properties as well as different architectures and dynamics. The associated technical details are problem specific and we discuss them in separate papers.

\begin{singlespace}
\bibliographystyle{plain}
\bibliography{nflgscompyxRefs}

\begin{thebibliography}{10}

\bibitem{AAAABGLP23}
E.~Agliari, L.~Albanese, F.~Alemanno, A.~Alessandrelli, A.~Barra, F.~Giannotti,
  D.~Lotito, and D.~Pedreschi.
\newblock Dense hebbian neural networks: A replica symmetric picture of
  supervised learning.
\newblock {\em Physica A: Statistical Mechanics and its Applications},
  626:129076, 2023.

\bibitem{AmiGutSom85}
D.~J. Amit, H.~Gutfreund, and H.~Sompolinsky.
\newblock Storing infinite number of patterns in a spin glass model of neural
  networks.
\newblock {\em Phys. Rev. Letters}, 55:1530, 1985.

\bibitem{Chatterjee06}
S.~Chatterjee.
\newblock A generalization of the {L}indenberg principle.
\newblock {\em The Annals of Probability}, 34(6):2061--2076.

\bibitem{DHLUV17}
M.~Demircigil, J.~Heusel, M.~Lowe, S.~Upgang, and F.~Vermet.
\newblock On a model of associative memory with huge storage capacity.
\newblock {\em Journal of Statistical Physics}, 168(2):288--299, July 2017.

\bibitem{FST00}
J.~Feng, M.~Shcherbina, and B.~Tirozzi.
\newblock On the critical capacity of the {H}opfield model.
\newblock {\em Communications in Mathematical Physics}, 2000.

\bibitem{GarMult87}
E.~Gardner.
\newblock Multiconnected neural network models.
\newblock {\em Journal of Physics A: Mathematical and General},
  20(1):3453--3464, August 1987.

\bibitem{Hebb49}
D.~O. Hebb.
\newblock Organization of behavior.
\newblock {\em New York: Wiley}, 1949.

\bibitem{Hop82}
J.~J. Hopfield.
\newblock Neural networks and physical systems with emergent collective
  computational abilities.
\newblock {\em Proc. Nat. Acad. Science}, 79:2554, 1982.

\bibitem{KomPat88}
J.~Komlos and R.~Paturi.
\newblock Convergence results in an autoassociative memory model.
\newblock {\em Neural Networks}, 1:239--250, 1988.

\bibitem{KroHop16}
D.~Krotov and J.~J. Hopfield.
\newblock Dense associative memory for pattern recognition.
\newblock {\em Advances in Neural Information Processing Systems},
  1:1180--1188, 2016.

\bibitem{Louk94}
D.~Loukianova.
\newblock Capacite de memoire dans le modele de {H}opfield.
\newblock {\em C. R. Acad. Sci. Paris t. 318, Serie I}, pages 157--160, 1994.

\bibitem{Louk94a}
D.~Loukianova.
\newblock Etude rigoureuse du modele de {H}opfield de memoire associative.
\newblock {\em These de doctorat de l'{U}niversite {P}aris 7}, 1994.

\bibitem{Louk97}
D.~Loukianova.
\newblock Lower bounds on the restitution error in the {H}opfield model.
\newblock {\em Probab. Theory Related Fields}, 107:161--176, 1997.

\bibitem{LucMez24}
Carlo Lucibello and Marc M\'ezard.
\newblock Exponential capacity of dense associative memories.
\newblock {\em Phys. Rev. Lett.}, 132:077301, Feb 2024.

\bibitem{MPRV87}
R.~J. MacEliece, E.~C. Posner, E.~Rodemich, and S.S. Venkatesh.
\newblock The capacity of the {H}opfield associative memory.
\newblock {\em IEEE Trans. Inform. Theory}, 33:461--482, 1987.

\bibitem{Newman88}
Ch.~M. Newman.
\newblock Memory capacity and neural network models: {R}igorous lower bounds.
\newblock {\em Neural Networks}, 1:223--238, 1988.

\bibitem{PasFig77}
L.~Pastur and A.~Figotin.
\newblock Exactly soluble model of a spin-glass.
\newblock {\em Soviet {J}. of {L}ow {T}emperature {P}hys.}, 3(378-383), 1977.

\bibitem{Rametal21}
H.~Ramsauer, B.~Schafl, J.~Lehner, P.~Seidl, M.~Widrich, L.~Gruber,
  M.~Holzleitner, T.~Adler, D.~Kreil, M.~K. Kopp, G.~Klambauer,
  J.~Brandstetter, and S.~Hochreiter.
\newblock Hopfield networks is all you need.
\newblock {\em In International Conference on Learning Representations}, 2021.

\bibitem{SteKuh94}
H.~Steffan and R.~Kuhn.
\newblock Replica symmetry breaking in attractor neural network models.
\newblock {\em Z. Phys. B}, 95:249--260, 1994.

\bibitem{StojnicGardGen13}
M.~Stojnic.
\newblock Another look at the {G}ardner problem.
\newblock 2013.
\newblock available online at \bl{\url{http://arxiv.org/abs/1306.3979}}.

\bibitem{StojnicDiscPercp13}
M.~Stojnic.
\newblock Discrete perceptrons.
\newblock 2013.
\newblock available online at \bl{\url{http://arxiv.org/abs/1303.4375}}.

\bibitem{StojnicLiftStrSec13}
M.~Stojnic.
\newblock Lifting $\ell_1$-optimization strong and sectional thresholds.
\newblock 2013.
\newblock available online at \bl{\url{http://arxiv.org/abs/1306.3770}}.

\bibitem{StojnicMoreSophHopBnds10}
M.~Stojnic.
\newblock Lifting/lowering {H}opfield models ground state energies.
\newblock 2013.
\newblock available online at \bl{\url{http://arxiv.org/abs/1306.3975}}.

\bibitem{StojnicGardSphNeg13}
M.~Stojnic.
\newblock Negative spherical perceptron.
\newblock 2013.
\newblock available online at \bl{\url{http://arxiv.org/abs/1306.3980}}.

\bibitem{StojnicGardSphErr13}
M.~Stojnic.
\newblock Spherical perceptron as a storage memory with limited errors.
\newblock 2013.
\newblock available online at \bl{\url{http://arxiv.org/abs/1306.3809}}.

\bibitem{Stojnicsflgscompyx23}
M.~Stojnic.
\newblock Bilinearly indexed random processes -- {\emph{stationarization}} of
  fully lifted interpolation.
\newblock 2023.
\newblock available online at \bl{\url{http://arxiv.org/abs/2311.18097}}.

\bibitem{Stojnicbinperflrdt23}
M.~Stojnic.
\newblock Binary perceptrons capacity via fully lifted random duality theory.
\newblock 2023.
\newblock available online at \bl{\url{http://arxiv.org/abs/2312.00073}}.

\bibitem{Stojnicnegsphflrdt23}
M.~Stojnic.
\newblock {Fl RDT} based ultimate lowering of the negative spherical perceptron
  capacity.
\newblock 2023.
\newblock available online at \bl{\url{http://arxiv.org/abs/2312.16531}}.

\bibitem{Stojnicnflgscompyx23}
M.~Stojnic.
\newblock Fully lifted interpolating comparisons of bilinearly indexed random
  processes.
\newblock 2023.
\newblock available online at \bl{\url{http://arxiv.org/abs/2311.18092}}.

\bibitem{Stojnicflrdt23}
M.~Stojnic.
\newblock Fully lifted random duality theory.
\newblock 2023.
\newblock available online at \bl{\url{http://arxiv.org/abs/2312.00070}}.

\bibitem{Stojnichopflrdt23}
M.~Stojnic.
\newblock Studying {H}opfield models via fully lifted random duality theory.
\newblock 2023.
\newblock available online at \bl{\url{http://arxiv.org/abs/2312.00071}}.

\bibitem{Stojnictcmspnncapdinfdiffactrdt23}
M.~Stojnic.
\newblock Exact capacity of the \emph{wide} hidden layer treelike neural
  networks with generic activations.
\newblock 2024.
\newblock available online at \bl{\url{http://arxiv.org/abs/2402.05719}}.

\bibitem{Tal98}
M.~Talagrand.
\newblock Rigorous results for the {H}opfield models with many patterns.
\newblock {\em Prob. Theor. Rel. Fields}, 110:109--176, 1998.

\bibitem{Vasetal17}
A.~Vaswani, N.~Shazeer, N.~Parmar, J.~Uszkoreit, L.~Jones, A.~N. Gomez,
  L.~Kaiser, and I.~Polosukhin.
\newblock Attention is all you need.
\newblock {\em In Advances in neural information processing systems},
  1:5998--6008, 2017.

\end{thebibliography}
\end{singlespace}

\end{document}